\documentclass{article}

\usepackage[utf8]{inputenc} 
\usepackage[T1]{fontenc}    
\usepackage{hyperref}       
\usepackage{url}            
\usepackage{booktabs}       
\usepackage{amsfonts}       
\usepackage{nicefrac}       
\usepackage{microtype}      
\usepackage{natbib}
\usepackage{amsbsy}
\usepackage{amsmath}
\usepackage{amssymb}
\usepackage{amsthm}
\usepackage{comment}
\usepackage{etoolbox}
\usepackage[mathscr]{euscript}
\usepackage[margin=1.0in]{geometry}
\usepackage{gastex}
\usepackage{times}
\usepackage{upgreek}

\usepackage[algo2e,ruled]{algorithm2e}

\usepackage{mathtools}
\usepackage{caption}

\AtBeginEnvironment{algorithm2e}{%
  \captionsetup{margin={-\algomargin,\algomargin}}%
  \SetAlgoNoLine
  \SetAlgoNoEnd
  \SetAlgoCaptionLayout{centerline}
  \SetKwFor{ForEach}{for each}{do}{endfch}
  \SetKwInOut{Input}{input}\SetKwInOut{Output}{output}
  \SetKwInOut{Return}{return}
  \SetKwInOut{Play}{play} \SetKwInOut{Receive}{receive}
}

\DeclareMathOperator*{\E}{\mathbb E}
\DeclareMathOperator*{\argmax}{argmax}
\DeclareMathOperator*{\argmin}{argmin}
\DeclareMathOperator*{\supp}{supp}

\newcommand{\RETURN}{\mbox{\textbf{return} }}
\newcommand{\WITH}{\upshape \mbox{\textbf{with} }}

\newcommand{\sA}{\mathscr A}
\newcommand{\sB}{\mathscr B}
\newcommand{\sC}{\mathscr C}

\newcommand{\sF}{\mathscr F}

\newcommand{\sS}{\mathscr S}

\newcommand{\bbP}{\mathbb{P}}
\newcommand{\Nset}{\mathbb{N}}
\newcommand{\Rset}{\mathbb{R}}

\newcommand{\bfl}{{\mathbf l}}
\newcommand{\bN}{{\mathbf N}}

\newcommand{\bw}{{\mathbf w}}
\newcommand{\bx}{{\mathbf x}}
\newcommand{\by}{{\mathbf y}}
\newcommand{\bz}{{\mathbf z}}

\newcommand{\balpha}{{\boldsymbol \alpha}}
\newcommand{\bbeta}{{\boldsymbol \beta}}

\newcommand{\sfp}{{\mathsf p}}
\newcommand{\sfq}{{\mathsf q}}
\newcommand{\sfu}{{\mathsf u}}
\newcommand{\sfw}{{\mathsf w}}

\newcommand{\cA}{\mathcal A}
\newcommand{\cC}{\mathcal C}

\newcommand{\cF}{\mathcal F}

\newcommand{\cK}{\mathcal K}

\newcommand{\cO}{O}

\newcommand{\cW}{\mathcal W}

\newcommand{\Reg}{\mathrm{Reg}}

\newcommand{\h}{\widehat}

\newcommand{\e}{\epsilon}

\renewcommand{\phi}{\varphi}

\newcommand{\comments}[1]{}
\newcommand{\set}[2][]{#1 \{ #2 #1 \} }
\newcommand{\ignore}[1]{}

\makeatletter
\newcommand{\vast}{\bBigg@{4}}
\newcommand{\Vast}{\bBigg@{5}}
\makeatother

\newcommand{\dest}{\mathrm{dest}}
\newcommand{\src}{\mathrm{src}}
\newcommand{\lab}{\mathrm{lab}}
\newcommand{\weight}{\mathrm{weight}}

\newcommand{\AWM}{\textsc{AWM}}
\newcommand{\PBWM}{\textsc{PBWM}}

\hypersetup{
  colorlinks   = true,
  urlcolor     = blue,
  linkcolor    = blue,
  citecolor    = blue
}

\newtheorem{theorem}{Theorem}

\newtheorem{corollary}{Corollary}

\newtheorem*{rep@theorem}{\rep@title}
\newcommand{\newreptheorem}[2]{%
\newenvironment{rep#1}[1]{%
 \def\rep@title{#2 \ref{##1}}%
 \begin{rep@theorem}}%
 {\end{rep@theorem}}}

\newreptheorem{theorem}{Theorem}
\newreptheorem{lemma}{Lemma}
\newreptheorem{corollary}{Corollary}

\begin{document}

\title{Online Learning with Automata-based Expert Sequences}

\author{Mehryar Mohri\footnote{Courant Institute and Google Research}
\and
Scott Yang\footnote{Courant Institute}}

\maketitle

\begin{abstract}
  We consider a general framework of online learning with expert
  advice where regret is defined with respect to sequences of
  experts accepted by a weighted automaton. Our framework covers
  several problems previously studied, including competing against
  $k$-shifting experts. We give a series of algorithms for this
  problem, including an automata-based algorithm extending
  weighted-majority and more efficient algorithms based on the notion
  of failure transitions.  We further present efficient algorithms
  based on an approximation of the competitor automaton, in particular
  $n$-gram models obtained by minimizing the $\infty$-R\'enyi
  divergence, and present an extensive study of the approximation
  properties of such models. Finally, we also extend our algorithms and results
  to the framework of sleeping experts. \ignore{Finally, we describe the
  extension of our approximation methods to online convex optimization
  and a general mirror descent setting.}
\end{abstract}

\section{Introduction}
\label{sec:intro}
 
Online learning is a general model for sequential prediction. Within
that framework, the setting of prediction with expert advice has
received widespread attention \citep{LittlestoneWarmuth1994,
  CesaBianchiLugosi2006,CesaBianchiMansourStoltz2007}.  In this
setting, the algorithm maintains a distribution over a set of experts,
or selects an expert from an implicitly maintained distribution.  At
each round, the loss assigned to each expert is revealed. The
algorithm incurs the expected loss over the experts and then updates
its distribution on the set of experts.  Its objective is to minimize
its expected regret, that is the difference between its cumulative
loss and that of the best expert in hindsight.

However, this benchmark is only significant when the best expert in
hindsight is expected to perform well. When that is not the case, then
the learner may still play poorly. As an example, it may be that no
single baseball team has performed well over all seasons in the past
few years. Instead, different teams may have dominated over different
time periods. This has led to a definition of regret against the best
sequence of experts with $k$ shifts in the seminal work of
\cite{HerbsterWarmuth1998} on \emph{tracking the best expert}.  The
authors showed that there exists an efficient online learning
algorithm for this setting with favorable regret guarantees.

This work has subsequently been improved to account for broader expert
classes \citep{GyorgyLinderLugosi2012}, to deal with unknown
parameters \citep{MonteleoniJaakkola2003}, and has been further
generalized \citep{CesaBianchiGaillardLugosiStoltz2012, Vovk1999}.
Another approach for handling dynamic environments has consisted of
designing algorithms that guarantee small regret over any subinterval
during the course of play. This notion, coined as \emph{adaptive
  regret} by \cite{HazanSeshadhri2009}, has been subsequently
strengthened and generalized \citep{DanielyGonenShalevShwartz2015,
  AdamskiyKoolenChernovVovk2012}.  Remarkably, it was shown by
\cite{AdamskiyKoolenChernovVovk2012} that the algorithm designed by
\cite{HerbsterWarmuth1998} is also optimal for adaptive regret.
\cite{KoolenDeRooij2013} described a Bayesian framework for online
learning where the learner samples from a distribution of expert
sequences and predicts according to the prediction of that expert
sequence. They showed how the algorithms designed for $k$-shifting
regret, e.g.\ \citep{HerbsterWarmuth1998, MonteleoniJaakkola2003}, can
be interpreted as specific priors in this formulation.  There has also
been work deriving guarantees in the bandit setting when the losses
are stochastic \citep{BesbesGurZeevi2014, WeiHongLu2016}.

\begin{figure*}[t]
\vskip -.15in
\centering
\begin{tabular}{ccc}
\includegraphics[scale=0.4]{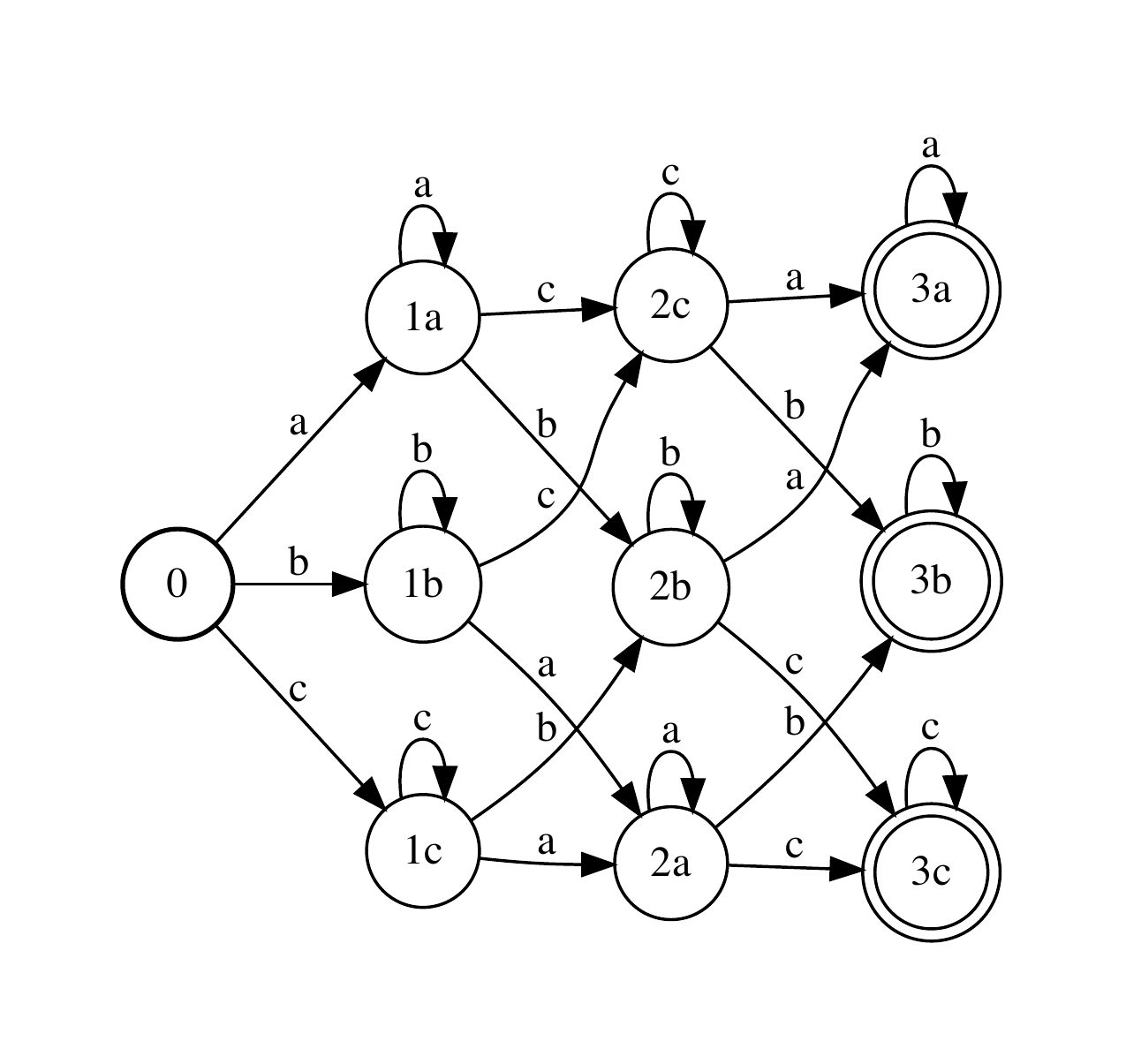} &
  {\includegraphics[scale=0.4]{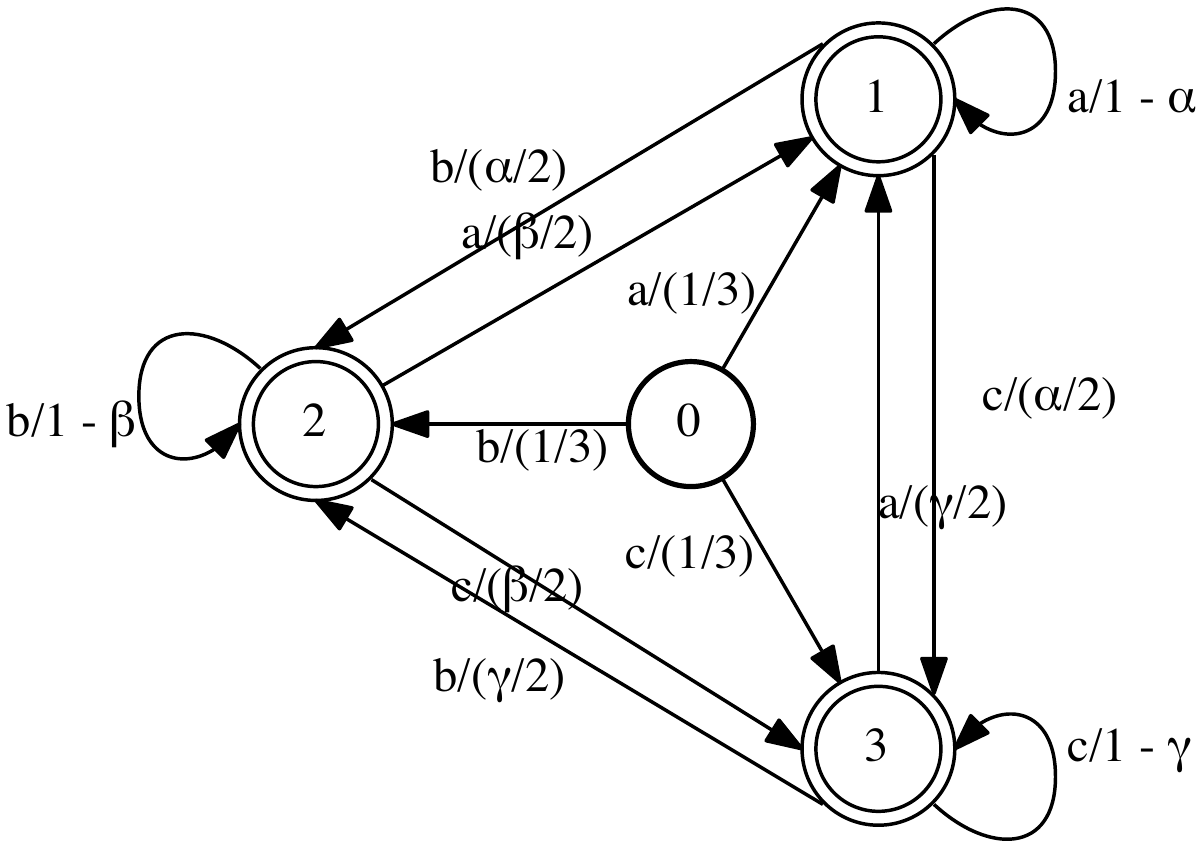}} &
  \includegraphics[scale=0.4]{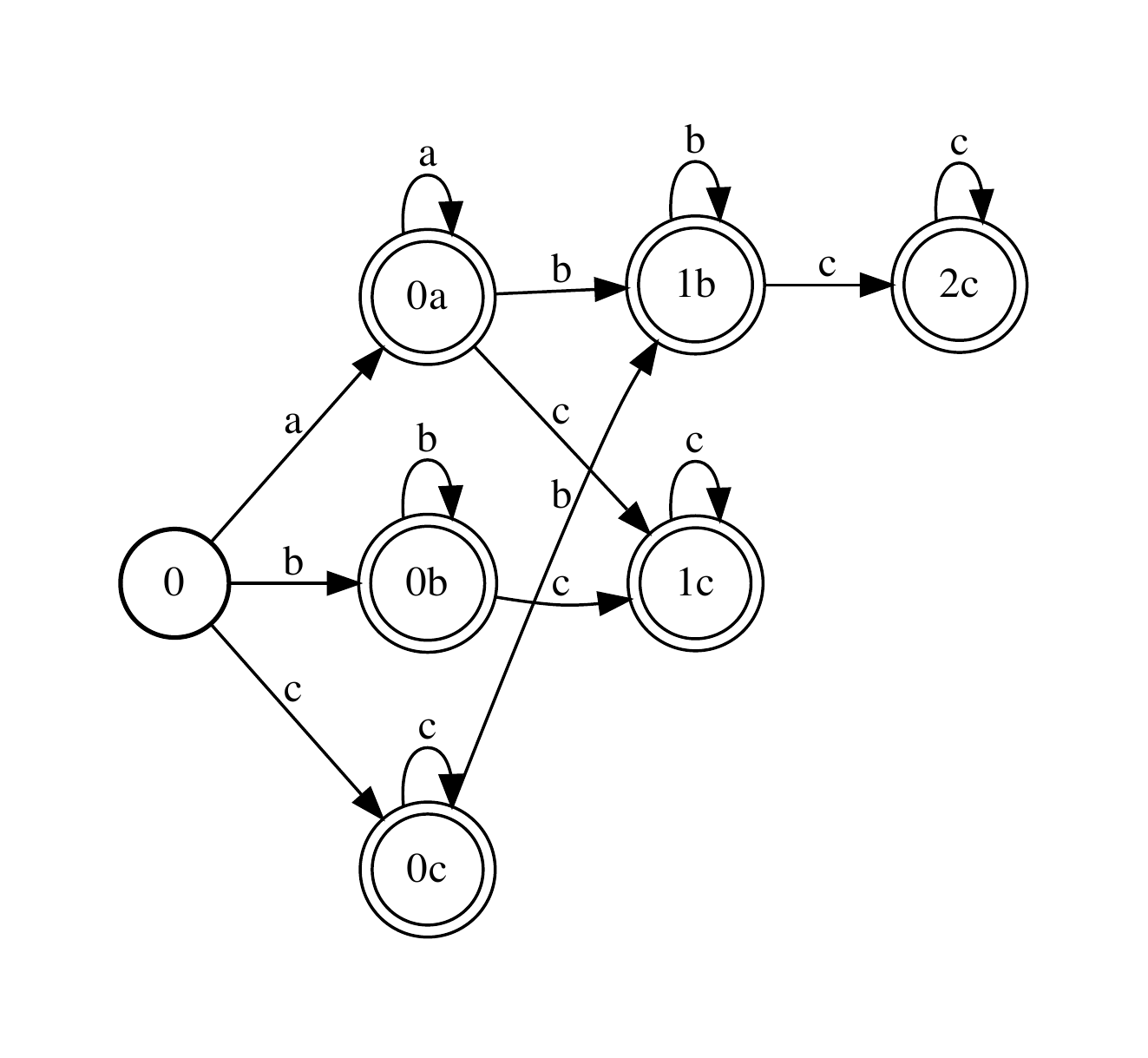} \\ 
  (i) &
  (ii) &
  (iii) \\ 
\end{tabular}
\caption{WFAs representing sequences of experts in
  $\Sigma = \set{a, b, c}$. (i) $\sC_\text{$k$-shift}$ with
  $k = 2$ shifts, all weights are equal to one and not indicated; (ii)
  $\sC_\text{weighted-shift}$ with $\alpha, \beta, \gamma \in [0, 1]$;
  (iii) $\sC_\text{hierarchy}$ a hierarchical family of expert 
  sequences: the learner must select expert $a$ from the start, can only shift 
  onto $b$ once, and can only shift onto $c$ twice. }
\label{fig:kshift}
\vskip -.15in
\end{figure*}

The general problem of online convex optimization in the presence of
non-stationary environments has also been studied by many
researchers. One perspective has been the design of algorithms that
maintain a guarantee against sequences that do not vary too much
\citep{MokhtariShahrampourJadbabaieRibeiro2016,
  ShahrampourJadbabaie2016, JadbabaieRakhlinShahrampourSridharan2015,
  BesbesGurZeevi2015}.  Another assumes that the learner has access to
a dynamical model that is able to capture the benchmark sequence
\citep{HallWillett2013}. \cite{GyorgySzepesvari2016} reinterpreted the
framework of \cite{HallWillett2013} to recover and extend the results
of \cite{HerbsterWarmuth1998}.

In this paper, we generalize the framework just described to the case
where the learner's cumulative loss is compared to that of sequences
accepted by a weighted finite automaton (WFA). This strictly
generalizes the notion of $k$-shifting regret, since $k$-shifting
sequences can be represented by an automaton (see
Figure~\ref{fig:kshift}), and further extends it to a notion of
\emph{weighted regret} which takes into consideration the sequence
weights. Our framework covers a very rich class of competitor classes,
including WFAs learned from past observations.  

Our contributions are mainly algorithmic but also include
several theoretical results and guarantees.  We first describe an
efficient online algorithm using automata operations that achieves
both favorable weighted regret and unweighted regret
(Section~\ref{sec:AWM}).  Next, we present and analyze more efficient
solutions based on an approximation of the WFA representing the set of competitor
sequences (Section~\ref{sec:approx}), including a specific analysis of
approximations via $n$-gram models both when minimizing the
$\infty$-R\'enyi divergence and the relative entropy.  Finally, we
extend the results above to the sleeping expert setting
\citep{FreundSchapireSingerWarmuth1997}, where the learner may not
have access to advice from every expert at each round
(Section~\ref{sec:sleep}).

\section{Learning setup}
\label{sec:setup}

We consider the setting of prediction with expert advice over
$T \in \Nset$ rounds. Let $\Sigma = \set{a_1, \ldots, a_N}$ denote a set
of $N$ experts.  At each round $t \in [T]$, an algorithm $\cA$ specifies
a probability distribution $\sfp_t$ over $\Sigma$, samples an expert
$i_t$ from $\sfp_t$, receives the vector of losses of all experts
$\bfl_t \in [0, 1]^N$, and incurs the specific loss $l_t[i_t]$.  A
commonly adopted goal for the algorithm is to minimize its static
(expected) regret $\Reg_T(\cA, \Sigma)$, that is the difference
between its cumulative expected loss and that of the best expert in
hindsight:
\begin{align}
\label{eq:staticregret}
\Reg_T(\cA, \Sigma) = 
\max_{x \in \Sigma} \sum_{t = 1}^T \sfp_t \cdot \bfl_t - \sum_{t = 1}^T l_t[x].
\end{align}
Here, we will consider an alternative benchmark, typically more
demanding, where the cumulative loss of the algorithm is compared
against the loss of the best sequence of experts
$\bx \in \Sigma^T$ among those accepted by a weighted
finite automaton (WFA) $\sC$ over the semiring
$(\Rset_+ \cup \set{+\infty}, +, \times, 0, 1)$.\footnote{Thus, the
  weights in $\sC$ are non-negative; the weight of a path is
  obtained by multiplying the transition weights along that path and
  the weight assigned to a sequence is obtained by summing the weights
  of all accepting paths labeled with that sequence.} The sequences
$\bx$ accepted by $\sC$ are those which are assigned a positive
value by $\sC$, $\sC(\bx) > 0$, which we will assume to be
non-empty. We will denote by $K \geq 1$ the cardinality of that set.

We will take into account the probability distribution $\sfq$ defined
by the weights assigned by $\sC$ to sequences of length $T$:
$\sfq(\bx) = \frac{\sC(\bx)}{\sum_{\bx \in \Sigma^T}
  \sC(\bx)}$. This leads to the following definition of
\emph{weighted regret} at time $T$ given a WFA $\sC$:
\begin{align}
\label{eq:weightedregret}
& \Reg_T(\cA, \sC) \\\nonumber
& = \max_{\substack{\bx \in \Sigma^T\\ \sC(\bx) > 0}} \set[\Bigg]{ \sum_{t = 1}^T \sfp_t
\cdot \bfl_t -  \sum_{t = 1}^T l_t[\bx[t]] 
  + \log [\sfq(\bx) K ] },
\end{align}
where $\bx[t]$ denotes the $t$th symbol of $\bx$.  The presence of the
factor $K$ only affects the regret definition by a constant additive
term $\log K$ and is only intended to make the last term
vanish when the probability distribution $\sfq$ is uniform, i.e.
$\sfq(\bx) = \frac{1}{K}$ for all $\bx$.  The last term in the
weighted regret definition can be interpreted as follows: for a given
value of an expert sequence loss $\sum_{t = 1}^T l_t[\bx[t]]$, the
regret is larger for sequences $\bx$ with a larger probability
$\sfq(\bx)$.  Thus, with this definition of regret, the learning
algorithm is pressed to achieve a small cumulative loss compared 
to expert sequences with small loss and high
probability. Notice that when $\sC$ accepts only constant sequences,
that is sequences $\bx$ with $\bx[1] = \ldots = \bx[T]$ and assigns
the same weight to them, then the notion of weighted regret coincides
with that of static regret (Formula~\ref{eq:staticregret}).

We also define the \emph{unweighted regret} $\Reg_T^0(\cA, \sC)$ of
algorithm $\cA$ at time $T$ given the WFA $\sC$ as:
\begin{align}
\label{eq:unweightedregret}
 \Reg_T^0(\cA, \sC) 
= \max_{\substack{\bx \in \Sigma^T\\ \sC(\bx) > 0}} \set[\Bigg]{
    \sum_{t = 1}^T \sfp_t \cdot \bfl_t - \sum_{t = 1}^T l_t[\bx[t]] }.
\end{align}
The weights of the WFA $\sC$ play no role in this notion of regret.
When $\sC$ has uniform weights, then the unweighted regret and
weighted regret coincide.

As an example, the sequences of experts with $k$ shifts studied by
\cite{HerbsterWarmuth1998} can be represented by the WFA
$\sC_\text{$k$-shift}$ of
Figure~\ref{fig:kshift}(i). Figure~\ref{fig:kshift}(ii) shows an
alternative weighted model of shifting experts, and
Figure~\ref{fig:kshift}(iii) shows a hierarchical family of expert
sequences.

\section{Automata Weighted-Majority algorithm}
\label{sec:AWM}

In this section, we describe a simple algorithm, \emph{Automata
  Weighted-Majority} (\AWM), that can be viewed as an enhancement of
the weighted-majority algorithm \citep{LittlestoneWarmuth1994} to the
setting of experts paths represented by a WFA. \footnote{This
  algorithm is in fact closer to the EXP4 algorithm
  \citep{AuerCesaBianchiFreundSchapire2002}.  However, EXP4 is
  designed for the bandit setting, so we use the weighted-majority
  naming convention.} We will show that it benefits from favorable
weighted and unweighted regret guarantees.

As with standard weighted-majority, \AWM\ maintains a distribution
$\sfq_t$ over the set of expert sequences $\bx \in \Sigma^T$ accepted
by $\sC$ at any time $t$ and admits a learning parameter $\eta >
0$. The initial distribution $\sfq_1$ is defined in terms of the
distribution $\sfq$ induced by $\sC$ over $\Sigma^T$, and
$\sfq_{t + 1}$ is defined from $\sfq_t$ via an exponential update:
for all $\bx \in \Sigma^T, t \geq 1$,
\begin{align}
  & \sfq_1[\bx] = \frac{\sfq[\bx]^\eta}{\sum_{\bx \in \Sigma^T}
  \sfq[\bx]^\eta}, \nonumber \\
  & \sfq_{t + 1}[\bx] = \frac{e^{-\eta \, l_t[\bx[t]]}\sfq_t[\bx]}{\sum_{\bx
    \in \Sigma^T} e^{-\eta \, l_t[\bx[t]]}\sfq_t[\bx]},
\end{align}
where we denote by $\bx[t] \in \Sigma$ the $t$th symbol in $\bx$.
$\sfq_t$ induces a distribution $\sfp_t$ over the expert set $\Sigma$
defined for all $a \in \Sigma$ by
\begin{equation}
\sfp_t[a] = \frac{\sum_{\bx \in \Sigma^T} \sfq_t[\bx] 1_{\bx[t] =
    a}}{\sum_{a \in \Sigma} \sum_{\bx \in \Sigma^T} \sfq_t[\bx] 1_{\bx[t] = a}}.
\end{equation}
Thus, $\sfp_t[a]$ is obtained by summing up the $\sfq_t$-weights of
all sequences with the $t$th symbol equal to $a$ and normalization.
The distributions $\sfp_t$ define the \AWM\ algorithm. Note that the
algorithm cannot be viewed as weighted-majority with $\sfq$-priors on
expert sequences as $\sfq_1$ is defined in terms of $\sfq^\eta$.

The following regret guarantees hold for \AWM.

\begin{theorem}
\label{th:awm}
Let $\sfq$ denote the probability distribution over expert sequences
of length $T$ defined by $\sC$ and let $K$ denote the cardinality of
its support. Then, the following upper bound holds for the weighted
regret of \AWM:
\begin{align*}
  \Reg_T(\AWM, \sC) 
  & \leq \frac{\eta T}{8} + \frac{1}{\eta} \log \bigg[ K^\eta \sum_{\bx
    \in \Sigma^T} \sfq[\bx]^\eta \bigg] \\
    &\leq \frac{\eta T}{8} + \frac{1}{\eta} \log K. 
\end{align*}
Furthermore, when $K \geq 2$, for any $T > 0$, there exists
$\eta^* > 0$, decreasing as a function of $T$, such that:
\begin{align*}
    \Reg_T(\AWM, \sC) 
    & \leq \sqrt{\frac{T H_{\eta^*}(\sfq)}{2}} - H_{\eta^*}(\sfq) + \log
    K,
\end{align*}
where
$H_\eta(\sfq) = \frac{1}{1 - \eta} \log\left(\sum_{\bx \in \Sigma^T}
  \sfq[\bx]^\eta \right)$ is the $\eta$-R\'{e}nyi entropy of $\sfq$.  The
unweighted regret of \AWM\ can be upper-bounded as follows:
\begin{align*}
  \Reg_T^0(\AWM, \sC) 
  & \leq \frac{\eta T}{8} + \frac{1}{\eta} \log K.
\end{align*}

\end{theorem}

The proof is an extension of the standard proof for the
weighted-majority algorithm and is given in Appendix~\ref{app:awm}.
The bound in terms of the R\'enyi entropy shows that the regret
guarantee can be substantially more favorable than standard bounds of
the form $O(\sqrt{T \log K})$, depending on the properties of the
distribution $\sfq$. First, since the $\eta$-R\'enyi entropy is
non-increasing in $\eta$ \citep{VanErvenHarremos2014}, we have
$H_{\eta^*}(\sfq) \leq H_0(\sfq) = \log(|\supp(\sfq)|) \leq \log
K$. Second, if the distribution $\sfq$ is concentrated on a subset
$\Delta$ with a small cardinality, $|\Delta| \ll K$, that is
$\sum_{\bx \not \in \Delta} \sfq[\bx]^{\eta^*} < \e (1-\eta^*)
\sum_{\bx \in \Delta} \sfq[\bx]^{\eta^*}$ for some $\e > 0$ and for
$\eta^* < 1$, then, by Jensen's inequality, the following upper bound
holds:
\begin{align*}
  H_\eta^*(\sfq)
\ignore{
  & \leq \frac{1}{1 - \eta^*} \log \bigg( \sum_{\bx \in \Delta}
    \sfq[\bx]^{\eta^*\!} + \e (1 - \eta^*) \sum_{\bx \in \Delta}
  \sfq[\bx] \bigg) \\
}
  & \leq \frac{1}{1 - \eta^*} \log \bigg(\sum_{\bx \in \Delta}
    \sfq[\bx]^{\eta^*\!}\bigg) + \e \\ 
    &\leq \frac{1}{1 - \eta^*} \log \bigg(|\Delta| \bigg(\frac{1}{|\Delta|}\sum_{\bx \in \Delta}
    \sfq[\bx]\bigg)^{\eta^*\!}\bigg) + \e \\ 
    &\leq \log(|\Delta|) + \e. 
\end{align*}

{\bf Efficient algorithm}. We now present an efficient computation
of the distributions $\sfp_t$.  Algorithm~\ref{alg:awm} gives the
pseudocode of our algorithm.  We will assume throughout that $\sC$ is
deterministic, that is it admits a single initial state and no two
transitions leaving the same state share the same label.  We can
efficiently compute a WFA accepting the set of sequences of length $T$
accepted by $\sC$ by using the standard intersection algorithm for
WFAs (see Appendix~\ref{app:intersection} for more detail on this
algorithm). Let $\sS_T$ be a deterministic WFA accepting the
set of sequences of length $T$ and assigning weight one to each (see
Figure~\ref{fig:S_T}). Then, the intersection of $\sC$ and $\sS_T$ is
a WFA denoted by $\sC \cap \sS_T$ which, by definition, assigns to
each sequence $\bx \in \Sigma^T$ the weight
\begin{equation}
  (\sC \cap \sS_T) (\bx) = \sC(\bx) \times \sS_T(\bx) = \sC(\bx),
\end{equation}
and assigns weight zero to all other sequences. Furthermore, the WFA
$\sB = (\sC \cap \sS_T)$ returned by the intersection algorithm is
deterministic since both $\sC$ and $\sS_T$ are deterministic.  Next,
we replace each transition weight of $\sB$ by its $\eta$-power. Since
$\sB$ is deterministic, this results in a WFA that we denote by
$\sB^\eta$ and that associates to each sequence $\bx$ the weight
$\sC[\bx]^\eta$. Normalizing $\sB^\eta$ results in a WFA $\sA$
assigning weight
$\sA[\bx] = \frac{\sB[\bx]^\eta}{\sum_x \sB[\bx]^\eta} = \sfq_1[x]$ to
any $\bx \in \Sigma^T$.  This normalization can be achieved in time
that is linear in the size of the WFA $\sB^\eta$ using the
\textsc{Weight-Pushing} algorithm \citep{Mohri1997bis,Mohri2009}. For
completeness, we describe this algorithm in
Appendix~\ref{app:weightpush}. Note that since $\sB^\eta$ is acyclic,
its size is in $\cO(|E_\sA|)$.\footnote{The \textsc{Weight-Pushing}
  algorithm has been used in many other contexts to make a directed
  weighted graph stochastic. This includes network normalization in
  speech recognition \citep{MohriRiley2001}, and online learning with
  large expert sets
  \citep{TakimotoWarmuth2003,CortesKuznetsovMohriWarmuth2015}, where
  the resulting stochastic graph enables efficient sampling.  The
  problem setting, algorithms and objectives in the last two
  references are completely distinct from ours. In particular, (a) in
  those, each path of the graph represents a single expert, while in
  our case each path is a sequence of experts; (b) in those,
  weight-pushing is applied at every round, while in our case it is
  used once at the start of the algorithm; (c) the regret is with
  respect to a static expert, while in our case it is with respect to
  a WFA of expert sequences.}  We will denote by $\sA$ the resulting
WFA.

For any state $u$ of $\sA$, we will denote by $\bbeta[u]$ the sum of
the weights of all paths from $u$ to a final state. The vector
$\bbeta$ can be computed in time that is linear in the number of states and
transitions of $\sA$ using a simple \emph{single-source
  shortest-distance} algorithm in the semiring
$(\Rset_+ \cup \set{+\infty}, +, \times, 0, 1)$ \citep{Mohri2009}, or
the forward-backward algorithm. We call this subroutine \textsc{BwdDist} in 
the pseudocode.

\begin{algorithm2e}[t]
  \TitleOfAlgo{\AWM($\sC$, $\eta$)}
  $\sB \gets \sC \cap \sS_T$ \\
  $\sA \gets \textsc{Weight-Pushing}(\sB^\eta)$ \\
  $\bbeta \gets \textsc{BwdDist}(\sA)$ \\
  $\balpha \gets 0$; $\balpha[I_\sA] \gets 1$ \\ 
  \ForEach{$e \in E_{\sA}^{0 \to 1}$}{
      $\sfp_1[\lab[e]] \gets \weight[e]$.
    }
 \For{$t \gets 1$ \KwTo $T$}{
    $i_t \gets $\textsc{Sample}($\sfp_t$); \textsc{Play}($i_t$); \textsc{Receive}($\bfl_t$)\\
    $Z \gets 0$; $\bw \gets 0$\\ 
    \ForEach{$e \in E_{\sA}^{t \to t + 1}$}{
      $\weight[e] \gets \weight[e] \, e^{-\eta l_t[\lab[e]]}$ \\
      $\bw[\lab[e]] \gets \bw[\lab[e]] + \balpha[\src[e]] \, \weight[e] \, \bbeta[\dest[e]]$\\
      $Z \gets Z + \bw[\lab[e]]$\\
      $\balpha[\dest[e]] \gets \balpha[\dest[e]] + \balpha[\src[e]] \,
      \weight[e] $
    }
    $\sfp_{t + 1} \gets \frac{\bw}{Z}$ \\ 
  }
\caption{\textsc{AutomataWeightedMajority}(AWM).} 
\label{alg:awm}
\end{algorithm2e}

\begin{figure}[t]
\vskip -.15in
\centering
\includegraphics[scale=0.55]{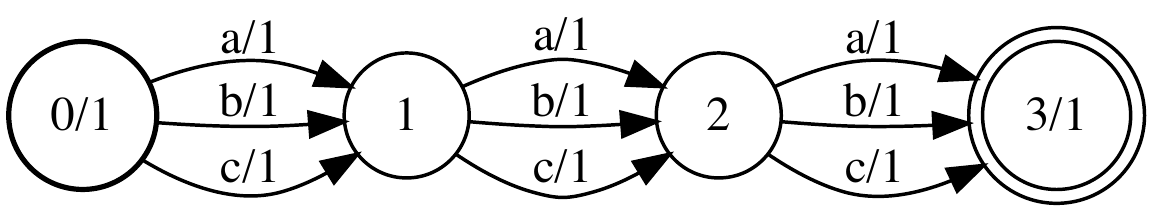}
\caption{WFA $\sS_T$, for $\Sigma = \set{a, b, c}$ and $T = 3$.}
\label{fig:S_T}
\vskip .1in
\centering
\includegraphics[scale=0.65]{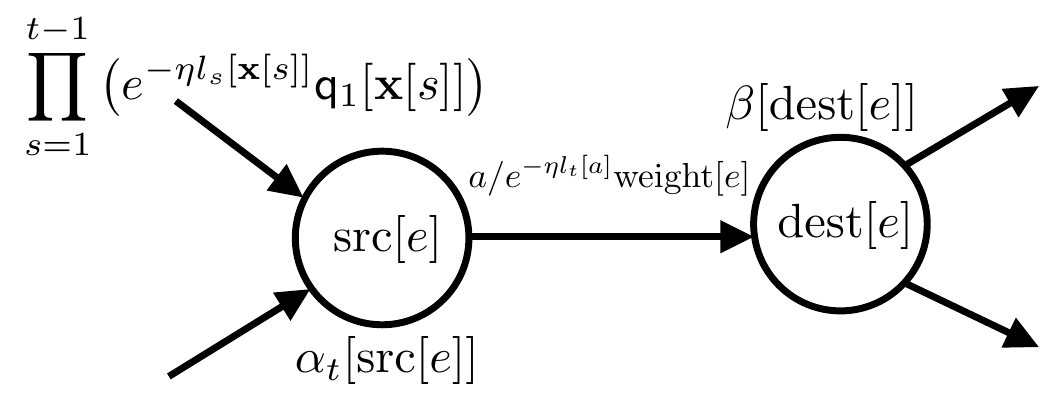}
\caption{Illustration of algorithm \AWM.}
\label{fig:awm}
\vskip -.15in
\end{figure}
 
We will denote by $Q_t$ the set of states in $\sA$ that can be reached
by sequences of length $t$ and by $E_{\sA}^{t \to t + 1}$ the set of
transitions from a state in $Q_t$ to a state in $Q_{t + 1}$.  For each
transition $e$, let $\src[e]$ denote its source state, $\dest[e]$ its
destination state, $\lab[e] \in \Sigma$ its label, and
$\weight[e] \geq 0$ its weight.  Since $\sA$ is normalized, the expert
probabilities $\sfp_1[a]$ for $a \in \Sigma$ can be read off the
transitions leaving the initial state: $\sfp_1[a]$ is the weight of
the transition in $E_{\sA}^{0 \to 1}$ labeled with $a$.

Let $\balpha_t[u]$ denote the \emph{forward weights}, that is the sum
of the weights of all paths from the initial state to state $u$ just
before the $t$th round.  At round $t$, the weight of each transition
$e$ in $E_{\sA}^{t \to t + 1}$ is multiplied by
$e^{-\eta l_t[\lab[e]]}$.  This results in new forward weights
$\balpha_{t + 1}[u]$ at the end of the $t$-th iteration.  $\balpha_{t + 1}$
can be straightforwardly derived from $\balpha_t$ since for
$u \in Q_{t + 1}$, $\balpha_{t + 1}[u]$ is given by
$\balpha_{t + 1}[u] = \sum_{e\colon \dest[e] = u} \balpha_t[\src[e]]
\weight[e]$. 

Observe that for any $t \in [T]$ and $\bx$, $\sfq_t[\bx]$ can be
written as follows by unwrapping its recursive update definition:
$$\sfq_t[\bx] = \frac{e^{-\eta \sum_{s = 1}^{t - 1} l_s[\bx[s]]}
  \sfq_1[\bx]}{\sum_{\bx \in \Sigma^T} e^{-\eta \sum_{s = 1}^{t - 1}
    l_s[\bx[s]]} \sfq_1[\bx]}.$$
    In view of that, for any
$a \in \Sigma$, $\sfp_{t + 1}[a]$ can be written as follows:
\begin{align*}
\sfp_{t + 1}[a] 
& = \frac{\sum_{\bx \in \Sigma^T} e^{-\eta \sum_{s = 1}^t
    l_s[\bx[s]]} \sfq_1[\bx] 1_{\bx[t] =
    a}}{\sum_{a \in \Sigma} \sum_{\bx \in \Sigma^T} e^{-\eta \sum_{s = 1}^t
    l_s[\bx[s]]} \sfq_1[\bx] 1_{\bx[t] =
    a}}.
\end{align*}
Since the WFA $\sA$ is deterministic, for any $\bx$ accepted by $\sA$
there is a unique accepting path $\pi$ in $\sA$ labeled with
$\bx$. The numerator of the expression of $\sfp_{t + 1}[a]$ is then
the sum of the weights of all paths in $\sA$ with the $t$th symbol $a$
at the end of $t$th iteration. This can be expressed as the sum over
all transitions $e$ in $E_{\sA}^{t \to t + 1}$ with label $a$ of the
total \emph{flow} through $e$, that is the sum of the weights of all
accepting path going through $e$:
$\balpha_t[\src[e]] \, \weight[e] \, \bbeta[\dest[e]]$ (see
Figure~\ref{fig:awm}).  This is precisely the formula 
determining $\sfp_{t + 1}$ in the pseudocode, where $Z$ is the
normalization factor.

The \AWM\ algorithm is closely related to the Expert Hidden Markov
Model of \cite{KoolenDeRooij2013} given for the log loss. It can be
viewed as a generalization of that algorithm to arbitrary loss
functions. A key difference between our setup and the perspective
adopted by \cite{KoolenDeRooij2013} is that they assume a Bayesian
setting where a prior distribution over expert sequences is given and
must be used. We assume the existence of a competitor automaton
$\sC$, but do not necessarily need to sample from it for making
predictions. This will be crucial in the next section, where we use a
different WFA than $\sC$ to improve computational efficiency while
preserving regret performance. Also,  
the prior distribution in \citep{KoolenDeRooij2013} would be over
$\sC_T$ (for a large $T$) and not $\sC$.

The computational complexity of \AWM\ at each round $t$ is
$\cO\big(|E_{\sA}^{t \to t + 1}|\big)$, that is the time to update the
weights of the transitions in $E_{\sA}^{t \to t + 1}$ and to
incrementally compute $\balpha$ for states reached by paths of length
$t + 1$. The total computational cost of the algorithm is thus
$\cO\big (\sum_{t = 1}^T |E_{\sA}^{t \to t + 1}| \big) =
\cO(|E_\sA|)$, where $E_\sA$ is the set of transitions of
$\sA$.\footnote{By the discussion above and
  Appendix~\ref{app:intersection}, the total complexity of the
  intersection and weight-pushing operations is also in
  $\cO(|E_\sA|)$, so that they do not add any additional
  cost. Moreover, these two operations need only be carried out once
  and can be performed offline.}  Note that $\sA$ and $\sC \cap \sS_T$
admit the same topology, thus the total complexity of the algorithm
depends on the number of transitions of the intersection WFA
$\sC \cap \sS_T$, which is at most $|\sC| NT$. This can be
substantially more favorable than a na\"{i}ve algorithm, whose
worst-case complexity is exponential in $T$.

When the number of transitions of the intersection WFA
$\sC \cap \sS_T$ is not too large compared to the number of experts
$N$, the \AWM\ algorithm is quite efficient.  However, it is natural
to ask whether one can design efficient algorithms even if the number
of transitions $E_{\sA}^{t \to t + 1}$ to process per round is large
(which may be the case even for a \emph{minimized} WFA
$\sC \cap \sS_T$ \citep{Mohri2009}).



We will give two sets
of solutions to derive a more efficient algorithm, which can be
combined for further efficiency. In the next section, we
discuss a solution that consists of using an
approximate WFA with a smaller number of transitions.
In Appendix~\ref{app:phiaut}, we show
that the notion of \emph{failure transition}, originally used in the
design of string-matching algorithms and recently employed for 
parameter estimation in backoff $n$-gram language models
\citep{RoarkAllauzenRiley2013}, can be used to derive a more compact representation
of the WFA $\sC \cap \sS_T$, thereby resulting in a significantly
more efficient online learning algorithm that still admits compelling regret
guarantees.

\section{Approximation algorithms}
\label{sec:approx}

In this section, we present approximation algorithms for the problem
of online learning against a weighted sequence of experts represented
by a WFA $\sC$.  Rather than using the intersection WFA
$\sC_T = \sC \cap \sS_T$, we will assume that \AWM\ is run with an
approximate WFA $\h \sC_T$. The main motivation for doing so is that
the algorithm can be substantially more efficient if $\h \sC_T$ admits
significantly fewer transitions than $\sC_T$. Of course, this comes at
the price of a somewhat weaker regret guarantee that we now analyze in
detail.
\ignore{In Appendix~\ref{app:timeindepapprox}, we will also discuss
how to perform approximations directly on the WFA $\sC$.}

\subsection{Effect of WFA approximation}
\label{sec:WFAapprox}

We first analyze the effect of automata approximation on the regret of \AWM.
As in the previous section, we denote by $\sfq$ the distribution
defined by $\sC_T$ over sequences of length $T$. We will similarly
denote by $\h \sfq$ the distribution defined by $\h \sC_T$ over the
same set. The effect of the WFA approximation on the regret can be
naturally expressed in terms of the $\infty$-R\'{e}nyi divergence
$D_\infty(\sfq \| \h \sfq)$ between the distributions $\sfq$ and
$\h \sfq$, which is defined by
$D_\infty(\sfq \| \h \sfq) = \sup_{\bx \in \Sigma^T} \log [
\sfq(\bx)/\h \sfq(\bx) ]$\ignore{ \citep{Renyi1961}}.

\begin{theorem}
\label{th:WFAapprox}
The weighted regret of the \AWM\ algorithm with respect to the WFA
$\sC$ when run with $\h \sC_T$ instead of $\sC_T$ can be
upper bounded as follows:
\begin{align*}
  \Reg_T(\cA, \sC) 
  & \leq \frac{\eta T}{8} + \frac{1}{\eta} \log \Big[K^\eta \sum_{\bx} \h
  \sfq[\bx]^\eta \Big] + D_\infty(\sfq \| \h \sfq) \\ 
  & \leq \frac{\eta T}{8} + \frac{1}{\eta} \log K + D_\infty(\sfq \| \h \sfq). 
\end{align*}
Its unweighted regret can be upper bounded as follows:
\begin{align*}
  \Reg_T^0(\cA, \sC) &\leq \max_{\sC(\bx) > 0} \frac{\eta T}{8} + \frac{1}{\eta} \log \bigg[
    \frac{1}{\sfq[\bx]}\bigg] + \frac{1}{\eta}
D_\infty(\sfq \| \h \sfq).
\end{align*}
\end{theorem}
The proof is given in Appendix~\ref{app:WFAapprox}.
Theorem~\ref{th:WFAapprox} shows that the extra cost of using an
approximate WFA $\h \sC_T$ instead of $\sC_T$ is
$D_\infty(\sfq \| \h \sfq)$ for the weighted regret and similarly
$\frac{1}{\eta} D_\infty(\sfq \| \h \sfq)$ for the unweighted regret.
The bound is tight since the best sequence in hindsight in the regret
definition may also be the one maximizing the log-ratio.

The theorem suggests a general algorithm for selecting an approximate
WFA $\h \sC$ out of a family $\cC$ of WFAs with a relatively small
number of transitions. This consists of choosing $\h \sC$ to
minimize the R\'enyi divergence as defined by the following program:
\begin{equation}
\label{eq:autapprox}
\min_{\h \sC \in \cC} D_\infty( \sfq \| \h \sfq),
\end{equation}
where $\h q$ is the distribution induced by $\h \sC$ over $\Sigma^T$
(the one obtained by computing $\h \sC_T = \h \sC \cap \sS_T$ and
normalizing the weights). The theorem ensures that the solution
benefits from the most favorable regret guarantee among the WFAs in
$\cC$.  When the set of distributions associated to $\cC$ is convex,
then the set of distributions defined over $\Sigma^T$ is also
convex. This is then a convex optimization problem, since
$\h \sfq \mapsto \log(\sfq/\h \sfq)$ is a convex function and the
supremum of convex functions is convex.

The choice of the family $\cC$ is subject to a trade-off: approximation
accuracy versus computational efficiency of using WFAs in $\cC$.
This raises a model selection question for which we discuss in
detail a solution in Section~\ref{sec:min-Renyi}:
given a sequence of families $(\cC_n)_{n \in \bN}$ with growing
complexity and computational cost, the problem consists of selecting
the best $n$.

In the following, we will consider the case where the family $\cC$ of
weighted automata is that of \emph{$n$-gram models}, for which we can
upper bound the computational complexity.

\subsection{Minimum R\'{e}nyi divergence $n$-gram models}
\label{sec:min-Renyi}



Let $\Sigma^{\leq n - 1}$ denote the set of sequences of length at
most $n - 1$.  An $n$-gram language model is a Markovian model of
order $(n - 1)$ defined over $\Sigma^*$, which can be compactly
represented by a WFA with each state identified with a sequence
$\bx \in \Sigma^{\leq n - 1}$, thereby encoding the sequence
\emph{just read} to reach that state.  The WFA admits a transition
from state $(\bx[1] \cdots \bx[n - 1])$ to state
$(\bx[2] \cdots \bx[n - 1] a)$ with weight
$\sfw \big[a \, | \, \bx[1] \cdots \bx[n - 1] \big]$, for any
$a \in \Sigma$, and, for any $k \leq n - 1$, a transition from state
$(\bx[1] \cdots \bx[k - 1])$ to state $(\bx[1] \cdots \bx[k - 1] a)$
with weight $\sfw \big[ a \, | \, \bx[1] \cdots \bx[k - 1] \big]$, for
any $a \in \Sigma$. It admits a unique initial state which is the one
labeled with the empty string $\e$ (sequence of length zero) and all
its states are final. The WFA is stochastic, that is outgoing
transition weights sum to one at every state: thus,
$\sum_{a \in \Sigma} \sfw[a | \bx ] = 1$ for all
$\bx \in \Sigma^{\leq n - 1}$. Notice that this WFA is also
deterministic since it admits a unique initial state and no two
transitions with the same label leaving any
state. Figure~\ref{fig:bigram} illustrates this definition in the case
of a simple bigram model.\ignore{\footnote{Notice that $n$-gram models of this
  form are commonly used in language and speech processing. In these
  applications, the models are typically \emph{smoothed} since they
  are trained on a finite sample. In our case, smoothing will not be
  needed since we can train directly on $\sC_T$, which we interpret as
  the full language.}}

Note that the transition weights $\sfw[a | \bx]$, with $a \in \Sigma$
and $\bx \in \cup_{k \leq n - 1} \Sigma^k$ fully specify an $n$-gram
model. Since for a fixed $\bx \in \cup_{k \leq n - 1} \Sigma^k$,
$\sfw[\cdot | \bx]$ is an element of the simplex, an $n$-gram model
can be viewed as an element of the product of
$\sum_{k = 0}^{n - 1} \Sigma^k$ simplices, a convex set. We will denote
by $\cW_n$ the family of all $n$-gram models.

\begin{figure}[t]
\vskip -.15in
\centering
\includegraphics[scale=0.5]{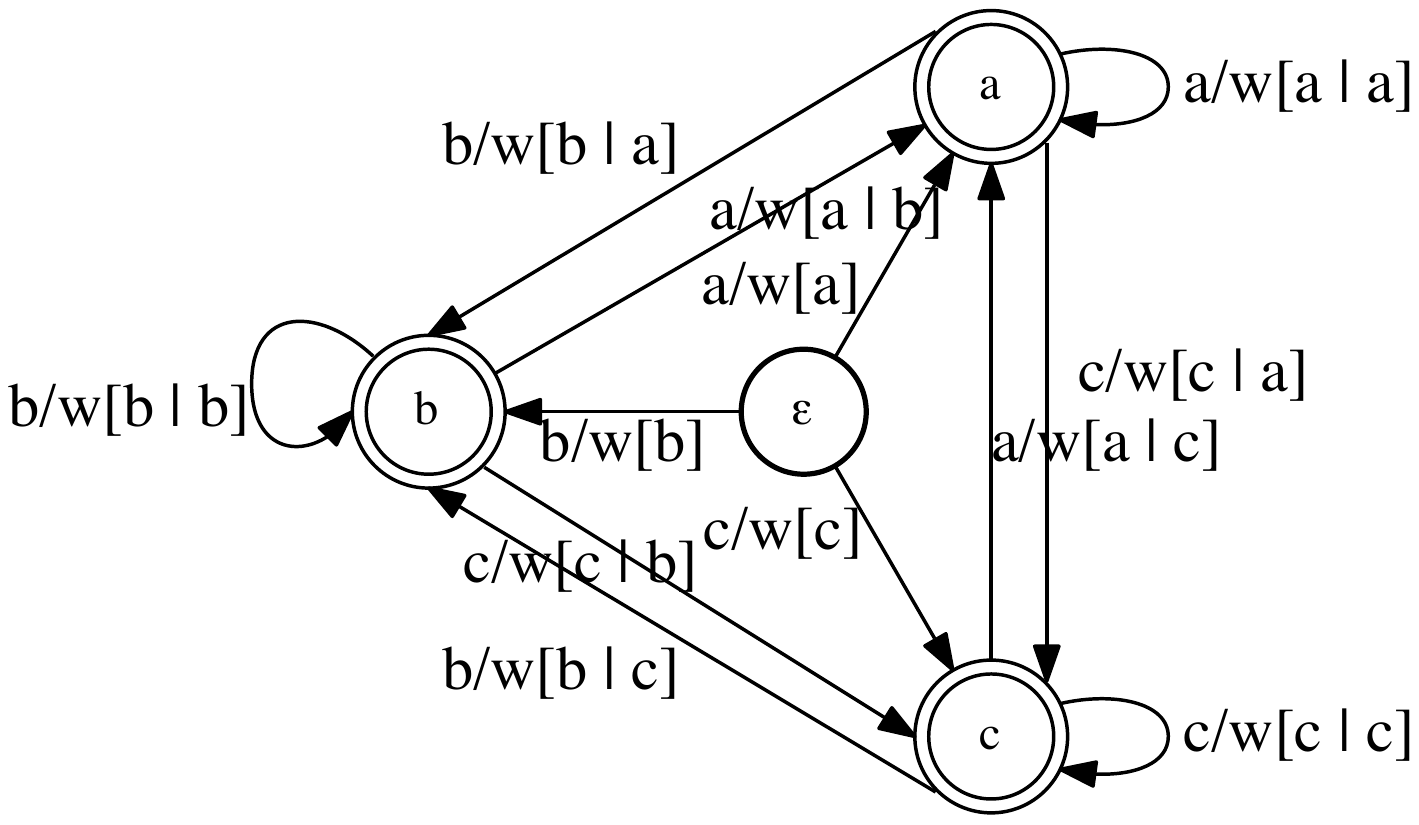} 
  \caption{A bigram language 
  model over the alphabet $\Sigma =\set{a, b, c}$.
 }
\label{fig:bigram}
\vskip -.15in
\end{figure}

One key advantage of $n$-gram models in this context is that the
per-iteration complexity can be bounded in terms of the number of
symbols. Since an $n$-gram model has at most $|\Sigma|^{n - 1}$
states, its per-iteration computational cost is in
$\cO\big(|\Sigma|^{n}\big)$ as each state can take one of $|\Sigma|$
possible transitions. For $n$ small, this can be very advantageous
compared to the original $\sC_T$, since in general the maximum
out-degree of states reached by sequences of length $t$ in the latter
can be very large. For instance, the automaton $\sC_\text{weighted-shift}$ 
in Figure~\ref{fig:kshift} (ii) can itself be viewed as a bigram model and 
admits efficient computation.   

For $n$-gram models, our approximation algorithm
(Problem~\ref{eq:autapprox}) can be written as follows:
\begin{equation}
\label{eq:minrenyingram}
\min_{\sfw \in \cW_n} D_\infty(\sfq \| \sfq_{\sfw})
=  \min_{\sfw \in \cW_n} \sup_{\bx \in \Sigma^T} 
\log\bigg[ \frac{\sfq[\bx]}{\sfq_{\sfw}[\bx]} \bigg],
\end{equation}
where $\sfq_{\sfw}$ is the distribution induced by the $n$-gram model
$\sfw$ on sequences in $\Sigma^T$. By definition of the $n$-gram
model, for any $\bx \in \Sigma^T$, $\sfq_{\sfw}[\bx]$ is given by the
following:
\begin{equation*}
\sfq_{\sfw}[\bx] = \prod_{t = 1}^T \sfw \big[\bx[t] \big| \bx_{\max(t - n +
  1, 1)}^{t - 1} \big],
\end{equation*}
since the weights of sequences of any fixed
length sum to one in an $n$-gram model. Problem~\ref{eq:minrenyingram} is a convex
optimization problem over $\cW_n$. The problem can be solved
using as an an extension of the Exponentiated Gradient (EG) algorithm
of \cite{KivinenWarmuth1997}, which we will refer to as
\textsc{Prod-EG}. The pseudocode of \textsc{Prod-EG}, a general convergence
guarantee, and its convergence guarantee in
the specific case of $n$-gram models are given in detail in
Appendix~\ref{app:Prod-EG} as Algorithm~\ref{alg:prodeg}, Theorem~\ref{th:prodeg},
and Corollary~\ref{cor:prodegngram} respectively.

{\bf Model selection}. In practice, we seek an $n$-gram model that
balances the tradeoff between approximation error and computational
cost. Assume that we are given a maximum per-iteration computational
budget $B$. We therefore wish to determine an $n$-gram approximation
model affordable within our budget and with the most favorable regret
guarantee.  Let $F(\sfq, \sfq_\sfw)$ denote the objective function of
Problem~\eqref{eq:minrenyingram}:
$F(\sfq, \sfq_\sfw) = D_\infty(\sfq \| \sfq_{\sfw})$.  By the
convergence guarantee of Corollary~\ref{cor:prodegngram}, if
$\sfq_\sfw$ is the $n$-gram model returned by \textsc{Prod-EG} after
$\tau$ iterations, we can write
$F(\sfq, \sfq_\sfw) - F(\sfq, \sfq_{\sfw^*}) \leq \Delta(\tau, n)$,
where $\sfw^*$ is the $n$-gram model minimizing
Problem~\eqref{eq:minrenyingram} over $\cW_n$ and $\Delta(\tau, n)$
the upper bound given by Corollary~\ref{cor:prodegngram}.  Thus, if
$F(\sfq, \sfq_\sfw) - \Delta(\tau, n) > \sqrt{T}$ for some $n$, then,
even the optimal $n$-gram model for this $n$ will cause an increase in
the regret.

Let $n^*$ be the smallest $n$ such that
$F(\sfq, \sfq_\sfw) - \Delta(\tau, n) \leq \sqrt{T}$ (or the smallest
value that exceeds our budget).  We can find this value in $\log(n^*)$
time using a two-stage process. In the first stage, we double $n$
after every violation until we find an upper bound on $n^*$, which we
denote by $n_{\text{max}}$. In the second stage, we perform a binary
search within $[1, n_{\text{max}}]$ to determine $n^*$. Each stage
takes $\log(n^*)$ iterations, and each iteration is the cost of
running \textsc{Prod-EG} for that specific value of $n$.  Thus, the
overall complexity of the algorithm is
$\cO\left(\log(n^*) \, \text{Cost}(\textsc{Prod-EG})\right)$, where
$\text{Cost}(\textsc{Prod-EG})$ is the cost of a call to
\textsc{Prod-EG}. The full pseudocode of this algorithm, 
\textsc{$n$-GramModelSelect}, is presented as Algorithm~\ref{alg:ngramselect},
where $\sfu_n$ denotes the
uniform $n$-gram model and $\textsc{Prod-EG-Update}(\sfq_\sfw, \cW_n)$
denotes one update made by \textsc{Prod-EG} when optimizing over
$\cW_n$.

\begin{algorithm2e}[t]
  \TitleOfAlgo{\textsc{$n$-GramModelSelect}($\sfq$, $\tau$, $B$)}
  $n \gets 1$;
  $\sfq_\sfw \gets \sfq_{\sfu_n}$;
  $s \gets 0$ \\
  \While{$s \leq \tau$}{
    $\sfq_\sfw \gets \textsc{Prod-EG-Update}(\sfq_\sfw, \cW_n)$  \\
    $s \gets s + 1$ \\
    \If{$F(\sfq, \sfq_\sfw) - \Delta(s, n) > \sqrt{T}$ and $|\Sigma|^n \leq B$}{
      $n \gets 2n$;
      $s \gets 0$;
      $\sfq_\sfw \gets \sfq_{\sfu_n}$
    }
  }
  $n_{\max} \gets n$.\\
  $\sfq_\sfw \gets \textsc{BinarySearch}([1, n_\text{max}], F(\sfq, \sfq_\sfw) -
  \Delta(\tau, n) \leq \sqrt{T})$\\
  \RETURN{$\sfq_\sfw$}
\caption{\textsc{$n$-GramModelSelect}.}
\label{alg:ngramselect}
\end{algorithm2e}

In the simple case of a unigram automaton model over two symbols and when the
distribution $\sfq$ defined by the intersection WFA $\sC_T$ is uniform,
we can give an explicit form of the solution of Problem~\ref{eq:minrenyingram}.
The solution is obtained from the paths with the
smallest number of occurrences of each symbol, which can be
straightforwardly found via a shortest-path algorithm in linear time.

\begin{theorem}
\label{th:inftyrd1gram}
Assume that $\sC_T$ admits uniform weights over all paths and
$\Sigma = \set{a_1, a_2}$.  For $j \in \set{1, 2}$, let $n(a_j)$ be the smallest
number of occurrences of $a_j$ in a path of $\sC_T$.  For any $j \in
\set{1, 2}$, define
\begin{equation*}
\sfq[a_j] = \frac{\max\left\{1, \frac{n(a_j)}{T -
      n(a_j)}\right\}}{1+\max\left\{1, \frac{n(a_j)}{T -
      n(a_j)}\right\}}.
\end{equation*}

Then, the unigram model $\sfw \in \cW_1$ solution of
$\infty$-R\'{e}nyi divergence optimization problem is defined by
$\sfw[a_{j^*}] = \sfq[a_{j*}]$, $\sfw[a_{j'}] = 1 - \sfw[a_{j^*}]$,
with
$j^* = \argmax_{j \in \set{1, 2}} \ n(a_j) \log \sfq[a_j] + \left[T -
  n(a_j) \right] \log\big[ 1 - \sfq[a_j] \big]$.
\end{theorem}

The proof of this result is provided in Appendix~\ref{app:unigram}.

Theorem~\ref{th:inftyrd1gram} shows that the solutions of the
$\infty$-R\'enyi divergence optimization are based on the $n$-gram
counts of sequences in $\sC_T$ with ``high entropy''. This can be
very different from the maximum likelihood solutions, which are based
on the average $n$-gram counts.  For instance, suppose we are under
the assumptions of Theorem~\ref{th:inftyrd1gram}, and specifically,
assume that there are $T$ sequences in $\sC_T$. Assume that one of
the sequences has $\left(\frac{1}{2} + \gamma\right)T$ occurrences of
$a_1$ for some small $\gamma > 0$ and that the other $T-1$ sequences
have $T-1$ occurrences of $a_1$. Then,
$n(a_1) = \left(\frac{1}{2} + \gamma\right)T$, and the solution
of the $\infty$-R\'enyi divergence optimization problem is given by
$\sfq_\infty(a_1) = \frac{1 + 2\gamma}{2}$ and
$\sfq_\infty(a_2) = \frac{1 - 2\gamma}{2}$. On the other hand,
the maximum-likelihood solution would be
$\sfq_1(a_1) = 1 + \frac{\gamma}{T} - \frac{3}{2T} + \frac{1}{T^2} \approx 1$
and $\sfq_1(a_2) = \frac{3}{2T} - \frac{\gamma}{T} - \frac{1}{T^2} \approx 0$
for large $T$.

\subsection{Maximum-Likelihood $n$-gram models}
\label{subsec:mlapprox}

A standard method for learning $n$-gram models is via
Maximum-Likelihood, which is equivalent to minimizing the
relatively entropy between the target distribution $\sfq$ and the
language model, that is via
\begin{align}
  \label{eq:mlngram}
\min_{\sfw \in \cW_n} D(\sfq \| \sfq_{\sfw}),
\end{align}
where, $D(\sfq \| \sfq_{\sfw})$ denotes the relative entropy,
$D(\sfq \| \sfq_{\sfw}) = \sum_\bx \sfq[\bx]
\log\Big[\frac{\sfq[\bx]}{\sfq_\sfw[\bx]} \Big]$.  Maximum likelihood
$n$-gram solutions are simple. For standard text data, the weight of each
transition is the frequency of appearance of the corresponding
$n$-gram in the text. For a probabilistic $\sC_T$, the weight can be
similarly obtained from the expected count of the $n$-gram in the
paths of $\sC_T$, where the expectation is taken over the probability
distribution defined by $\sC_T$ and can be computed efficiently
\citep{AllauzenMohriRoark2003}. In general, the solution of this
optimization problem does not benefit from the guarantee of
Theorem~\ref{th:WFAapprox} since the $\infty$-R\'{e}nyi divergence is
an upper bound on the relative entropy.\ignore{\footnote{The relative entropy
  also coincides with the R\'{e}nyi divergence of order $1$ and the
  R\'{e}nyi divergence is non-decreasing function of the order.}}
However, in some cases, maximum likelihood solutions do benefit from
favorable regret guarantees.  In particular, as shown by the following
theorem, remarkably, the maximum-likelihood bigram approximation to
the $k$-shifting automaton coincides with the \textsc{Fixed-Share}
algorithm of \cite{HerbsterWarmuth1998} and benefits from a constant
approximation error. Thus, we can view and motivate the design of the
\textsc{Fixed-Share} algorithm as that of a bigram approximation of
the desired competitor automaton, which represents the family of
$k$-shifting sequences.

\begin{theorem}
\label{th:bigramkshift}
Let $\sC_T$ be the $k$-shifting automaton for some $k$.  Then, the
bigram model $\sfw_2$ obtained by minimizing relative entropy is
defined for all $a_1, a_2 \in \Sigma$ by
\begin{align*}
\sfp_{\sfw_2}[a_1a_2]
\! = \! \frac{\big[ 1 - \frac{k}{(T - 1)}\big] \, 1_{a_1  = a_2} 
+ \big[ \frac{k}{(T - 1)(N - 1)} \big] \, 1_{a_1 \neq a_2}}{N} .
\end{align*}
Moreover, its approximation error can be bounded by a constant
(independent of $T$):
\begin{equation*}
    D_\infty(\sfq \| \sfq_{\sfw_2}) \leq  - \log \big[ 1 - 2e^{-\frac{1}{12k}} \big].
\end{equation*}
\end{theorem}
The proof of the theorem as well as other details about
Maximum-Likelihood are given in Appendix~\ref{app:mlapprox}.
The proof technique is illustrative
because it reveals that the maximum likelihood $n$-gram model has low
approximation error whenever (1) the model's distribution is
proportional to the distribution of $\sC_T$ on $\sC_T$'s support and
(2) most of the model's mass lies on the support of $\sC_T$. When the
automaton $\sC_T$ has uniform weights, then condition (1) is satisfied
when the $n$-gram model is uniform on $\sC_T$. This is true whenever
all sequences in $\sC_T$ have the same set of $n$-gram counts, and
every permutation of symbols over these counts is a sequence that lies
in $\sC_T$, which is the case for the $k$-shifting
automaton. Condition (2) is satisfied when $n$ is large enough, which
necessarily exists since the distribution is exact for $n = T$. On the
other hand, note that a unigram approximation would have satisfied
condition (1) but not condition (2) for the $k$-shifting automaton.

To the best of our knowledge, this is the first framework that
motivates the design of \textsc{Fixed-Share} with a focus on
minimizing tracking regret. Other works that have recovered
\textsc{Fixed-Share} (e.g. \citep{CesaBianchiGaillardLugosiStoltz2012,
  KoolenDeRooij2013, GyorgySzepesvari2016}) have generally viewed the
algorithm itself as the main focus.

Our derivation of \textsc{Fixed-Share} also allows us to naturally
generalize the setting of standard $k$-shifting experts to
$k$-shifting experts with non-uniform weights. Specifically, consider
the case where $\sC_T$ is an automaton accepting up to $k$-shifts but
where the shifts now occur with probability
$\sfq[a_2 | a_1, a_1 \neq a_2] \neq \frac{1}{N-1} 1_{\{a_2 \neq
  a_1\}}$.  Since the bigram approximation will remain exact on
$\sC_T$, we recover the exact same guarantee as in
Theorem~\ref{th:bigramkshift}.

Maximum likelihood $n$-gram models can further benefit from our use of
failure transitions and the $\phi$-conversion algorithm presented in
Appendix~\ref{app:phiaut}. This can reduce the size of the automaton
and often dramatically improve its computational efficiency without
affecting its accuracy.

\section{Extension to sleeping experts}
\label{sec:sleep}

In many real-world applications, it may be natural for some experts to
abstain from making predictions on some of the rounds. For instance,
in a bag-of-words model for document classification, the presence of a
feature or subset of features in a document can be interpreted as an
expert that is awake. This extension of standard prediction with
expert advice is also known as the \emph{sleeping experts framework}
\citep{FreundSchapireSingerWarmuth1997}.  The experts are said to be
asleep when they are inactive and awake when they are active and available
to be selected. This framework is distinct from the
permutation-based definitions adopted in the studies in \citep{KleinbergNiculescuMizilSharma2010, KanadeMcMahanBryan2009, KanadeSteinke2014}.  

Formally, at each round $t$, the adversary chooses an awake set
$A_t \subseteq \Sigma$ from which the learner is allowed to query an
expert.  The algorithm then (randomly) chooses an expert $i_t$ from
$A_t$, receives a loss vector $l_t \in [0,1]^{|\Sigma|}$ supported on
$A_t$ and incurs loss $l_t[i_t]$. Since some experts may not be
available in some rounds, it is not reasonable to compare the loss
against that of the best static expert or sequence of experts. In
\citep{FreundSchapireSingerWarmuth1997}, the comparison is made
against the best fixed mixture of experts normalized at each round
over the awake set:
$\min_{\sfu \in \Delta_N}\sum_{t = 1}^T \frac{\sum_{a \in A_t} \sfu[a] l_t[a]
}{\sum_{a' \in A_t} \sfu[a']}$, where $\Delta_N$ is the $(N-1)$-dimensional 
simplex.

We extend the notion of sleeping experts to the path setting, so that
instead of comparing against fixed mixtures over experts, we compare
against fixed mixtures over the family of expert sequences. With some
abuse of notation, let $A_t$ also represent the automaton accepting
all paths of length $T$ whose $t$-th transition has label in $A_t$.
Thus, we want to design an algorithm that performs well with respect
to the following quantity:
$$\min_{\sfu \in \Delta_K}\sum_{t = 1}^T \frac{\sum_{\bx \in \sC_T \cap A_t} \sfu[\bx] l_t[\bx[t]] }{\sum_{\bx \in \sC_T \cap A_t} \sfu[\bx]},$$
where $K$ is the number of accepting paths of $\sC_T$.

This motivates the design of \textsc{AwakeAWM}, a 
path-based weighted majority algorithm that generalizes the algorithms
in \citep{FreundSchapireSingerWarmuth1997} to arbitrary families of
expert sequences.  Like \AWM, \textsc{AwakeAWM} maintains a set of
weights over all the paths in the input automaton. At each round $t$,
the algorithm performs a weighted majority-type update. However, it
normalizes the weights so that the total weight of the awake set
remains unchanged. This prevents the algorithm from ``overfitting'' to
experts that have been asleep for many rounds.  The pseudocode of this
algorithm and the proof of its accompanying guarantee,
Theorem~\ref{th:awakeawm}, 
are provided in Appendix~\ref{app:sleep}.

\begin{theorem}[Regret Bound for \textsc{AwakeAWM}]
  \label{th:awakeawm}
  Let $K$ denote the number of accepting paths of $\sC_T = \sC \cap \sS_T$,
  and for each $t \in [T]$, let $A_t\subseteq \Sigma$ denote the set of experts 
  that are awake at time $t$.
  Then for any distribution $\sfu\in \Delta_K$, \textsc{AwakeAWM} admits
  the following unweighted regret guarantee:
  \begin{align*}
    &\sum_{t = 1}^T \sum_{\bx \in \sC_T \cap A_t} \sfu[\bx] \E_{a \sim \sfp_t^{A_t}} [l_t[a]] - 
    \sum_{t=1}^T \sum_{\bx \in \sC_T \cap A_t} \sfu[\bx] l_t[\bx[t]] \\
    &\leq \frac{\eta}{8} \sum_{t = 1}^T \sfu(A_t) + \frac{1}{\eta} \log(K). 
  \end{align*}
\end{theorem}

As with \textsc{AWM}, \textsc{AwakeAWM} is an efficient algorithm
with a total computational cost that is linear in the number of transitions 
of $\sA$ (or equivalently, $\sC_T$). 
Moreover, as in the non-sleeping expert setting, we can further improve the
computational  complexity by applying $\phi$-conversion to arrive at a 
or $n$-gram approximation and then $\phi$-conversion.
All other improvements in the sleeping expert setting will similarly mirror
those for the non-sleeping expert algorithms.

\section{Conclusion}
\label{sec:conclusion}

We studied a general framework of online learning against a competitor
class represented by a WFA and presented a number of algorithmic
solutions for this problem achieving sublinear regret guarantees using
automata approximation and failure transitions.  We also extended our
algorithms and results to the sleeping experts framework
(Section~\ref{sec:sleep}).\ignore{and to the online convex optimization
setting (Appendix~\ref{app:oco}).}  Our results can be
straightforwardly extended to the adversarial bandit scenario using
standard surrogate losses based on importance weighting techniques and
to the case where more complex formal language families such as
(probabilistic) context-free languages over expert sequences are
considered.

\newpage
\bibliographystyle{abbrvnat} 
\bibliography{ola}
\newpage
\appendix
\onecolumn

\section{Intersection of WFAs}
\label{app:intersection}

The intersection of two WFAs $\sA_1$ and $\sA_2$ is a WFA denoted by
$\sA_1 \cap \sA_2$ that accepts the set of sequences accepted by both
$\sA_1$ and $\sA_2$ and is defined for all $\bx$ by
\begin{equation*}
(\sA_1 \cap \sA_2)(\bx) = \sA_1(\bx) \, \sA_2(\bx).
\end{equation*}
There exists a standard efficient algorithm for computing the
intersection WFA \citep{Mohri2009}. States of $\sA_1 \cap \sA_2$ are
identified with pairs of states $Q_1$ of $\sA_1$ and $Q_2$ of $\sA_2$:
$Q \subseteq Q_1 \times Q_2$, as are the set of initial and final
states. Transitions are obtained by matching pairs of transitions from
each weighted automaton and multiplying their weights following the
rule
\begin{equation*}
\Big(q_1 \stackrel{a/w_1}{\longrightarrow} q_1', \; q_2 \stackrel{a/w_2}{\longrightarrow} q_2' \Big)
\quad \Rightarrow  \quad (q_1, q_2) \stackrel{a/(w_1
  w_2)}{\longrightarrow} (q_1', q_2').
\end{equation*}

The worst-case space and time complexity of the intersection of two
deterministic weighted finite automata (WFA) is linear in the size of the automaton the algorithm
returns. In the worst case, this can be as large as the product of the sizes
of the WFA intersected (i.e. $\cO(|\sA_1| |\sA_2|)$, where $|\sA_1|$ is the
sum of the number of states and transitions of $\sA_1$ and similarly with
$|\sA_2|$.  This corresponds to the case where every transition
of $\sA_1$ can be paired up with every transition of $\sA_2$. In 
practice far fewer transitions can be matched.

Notice that when both $\sA_1$ and $\sA_2$ are deterministic,
then $\sA_1 \cap \sA_2$ is also deterministic since there
is a unique initial state (pair of initial states of each WFA)
and since there is at most one transition leaving $q_1 \in Q_1$
or $q_2 \in Q_2$ labeled with a given symbol $a \in \Sigma$.

In the case of $\sC \cap S_T$, the WFA returned
is $\sB$, which has the same size as $\sA$. $\sA$ has more transitions than
states since each state admits at least on outgoing transition, so its size is
dominated by its number of transitions. Therefore, the complexity of
intersection here is in $\cO(|E_{\sA}|)$, where $|E_{\sA}|$ is at most
$|\sC|  NT$.

\section{\textsc{Weight-Pushing} algorithm}
\label{app:weightpush}

Here, we briefly describe the \textsc{Weight-Pushing} algorithm for a
WFA $\sA$ in the context of this paper \citep{Mohri1997bis,Mohri2009}.
We denote by $Q_{\sA}$ the set of states of $\sA$, by $E_\sA$ the set
of transitions of $\sA$, by $I_\sA$ its initial state, by $F_\sA$ the
set of its final states, and by $\rho_\sA(q)$ the final weight at a
final state $q$ -- for the WFAs considered in this paper the final
weights are all equal to one.

For any state $q \in Q_{\sA}$, let $d[q]$ denote the sum of the
weights of all paths from $q$ to final states:
\begin{equation*}
d[q] = \sum_{\pi \in P(q, F_{\sA})} \weight[\pi] \, \rho(\dest[\pi]),
\end{equation*}
where $P(q, F_{\sA})$ denotes the set of paths from $q$ to a state in
$F_{\sA}$. For an acyclic WFA $\sA$, the weights $d[q]$ can be
computed in linear time in the size of $\sA$, that is in
$\cO(|Q_\sA| + |E_{\sA}|)$, or $\cO(|E_{\sA}|)$ when every state of
$\sA$ admits at least one outgoing or incoming transition. This can be
done using a general shortest-distance algorithm
\citep{Mohri1997bis,Mohri2009}.

The weight-pushing algorithm then consists of the following
steps.  For any transition $e \in E_{\sA}$ such that
$d[\src[e]] \neq 0$, we update its weight as follows:
\begin{equation*}
  \weight[e] \leftarrow d[\src[e]]^{-1} \, \weight[e] \, d[\dest[e]].
\end{equation*}
For any state $q \in F_\sA$ with $d[q] \neq 0$, we update its final
weight as follows:
\begin{equation*}
\rho_\sA[q] \leftarrow d[q]^{-1} \, \rho_\sA[q].
\end{equation*}
The resulting WFA is guaranteed to be stochastic (at any state $q$,
the sum of the weights of all outgoing transitions, and the final
weight if $q$ is final, is equal to one)
\citep{Mohri2009}. Furthermore, if $d[I_{\sA}] = 1$, that is if the
sum of the weights of all paths is one, then path weights are
preserved by this weight-pushing operation. Otherwise, the weights of
all paths starting at the initial state is divided by $d[I_{\sA}]$.

\section{Proof of Theorem~\ref{th:awm}}
\label{app:awm}

\begin{reptheorem}{th:awm}
Let $\sfq$ denote the probability distribution defined by
$\sC_T = \sC \cap \sS_T$ and let $K$ denote the number of
accepting paths of $\sC_T$. Then, the following upper bound holds
for the weighted regret of \AWM:
\begin{align*}
  \Reg_T(\AWM, \sC) 
  & \leq \frac{\eta T}{8} + \frac{1}{\eta} \log \bigg[ K^\eta \sum_{\bx
    \in \Sigma^T} \sfq[\bx]^\eta \bigg]
    \leq \frac{\eta T}{8} + \frac{1}{\eta} \log K. 
\end{align*}
Furthermore, when $K \geq 2$, for any $T > 0$, there exists
$\eta^* > 0$, decreasing function of $T$, such that:
\begin{align*}
    \Reg_T(\AWM, \sC) 
    & \leq \sqrt{\frac{T H_{\eta^*}(\sfq)}{2}} - H_{\eta^*}(\sfq) + \log
    K,
\end{align*}
where
$H_\eta(\sfq) = \frac{1}{1 - \eta} \log\left(\sum_{\bx \in \Sigma^T}
  \sfq[\bx] \right)$ is the $\eta$-R\'{e}nyi entropy of $\sfq$.  The
unweighted regret of \AWM\ can be upper-bounded as follows:
\begin{align*}
  \Reg_T^0(\AWM, \sC) 
  & \leq \frac{\eta T}{8} + \frac{1}{\eta} \log K.
\end{align*}

\end{reptheorem}

\begin{proof}
  We will use a standard potential-based argument. For any $t \geq 1$
  and sequence $\bx \in \Sigma^T$, let $w_t[\bx]$ denote the sequence
  weight defining $\sfq_t$ via normalization,
  $\sfq_t[\bx] = \frac{w_t[\bx]}{\sum_{\bx} w_t[\bx]}$, that is
  $w_1[\bx] = \sfq[\bx]^\eta$ and, for $t \geq 2$,
  $w_t[\bx] = w_1[\bx] e^{-\eta \sum_{s = 1}^{t - 1}
    l_s[\bx[s]]}$. Let $\Phi_t$ be the potential defined by
  $\Phi_t = \log \left(\sum_{\bx} w_t[\bx] \right)$ for $t \geq
  1$. Then, by Hoeffding's inequality, we can write
  \begin{align*}
    \Phi_{t + 1} - \Phi_t
    & = \log \left [\frac{\sum_{\bx} w_t[\bx] \, e^{-\eta l_t[\bx[t]]}}{\sum_{\bx} w_t[\bx]} \right]\\
    & = \log \left[ \E_{\bx \sim \sfq_t} \left[ e^{-\eta l_t[\bx[t]]} \right] \right] \\
    & \leq -\eta \E_{\bx \sim \sfq_t} \big[ l_t[\bx[t]] \big] + \frac{\eta^2}{8} 
    = -\eta \E_{a \sim \sfp_t} \big[ l_t[a] \big] + \frac{\eta^2}{8}.
  \end{align*}
  Summing up these inequalities over $t \in [1, T]$ results in the
  following upper bound:
\begin{equation*}
  \Phi_{T + 1} - \Phi_1 \leq -\eta \sum_{t = 1}^T \E_{a \sim \sfp_t} \left[
    l_t[a] \right] + \frac{\eta^2 T}{8}.
\end{equation*}
We can straightforwardly derive a lower bound for the same quantity
for any sequence $\bx_0 \in \Sigma^T$:
\begin{align*}
  \Phi_{T + 1} - \Phi_1
  & = \log \Big[ \sum_{\bx} w_{T + 1}[\bx] \Big] - \log \Big[ \sum_{\bx}
    w_1[\bx] \Big] \\
  & \geq \log [ w_{T + 1}[\bx_0] ] - \log \Big[ \sum_{\bx} w_1[\bx] \Big]\\
  & = - \eta \sum_{t = 1}^T l_t[\bx_0[t]] + \log[\sfq[\bx_0]^\eta ] - \log \Big[ \sum_{\bx} \sfq[\bx]^\eta \Big]. 
\end{align*}
Comparing the upper and lower bounds gives
\begin{align*}
  - \eta \sum_{t = 1}^T l_t[\bx_0[t]] + \log[\sfq[\bx_0]^\eta ] - \log \Big[ \sum_{\bx} \sfq[\bx]^\eta \Big]
  & \leq -\eta \sum_{t = 1}^T \E_{a \sim \sfp_t} \left[l_t[a] \right] + \frac{\eta^2 T}{8},
\end{align*}
which can be rearranged as
\begin{align*}
  & \sum_{t = 1}^T \E_{a \sim \sfp_t} \left[l_t[a] \right] - \sum_{t = 1}^T l_t[\bx_0[t]] 
  \leq \frac{\eta T}{8} - \log[\sfq[\bx_0] ] +
    \frac{1}{\eta} \log \Big[ \sum_{\bx} \sfq[\bx]^\eta \Big]\\
  \Leftrightarrow & \sum_{t = 1}^T \E_{a \sim \sfp_t} \left[l_t[a]
    \right] - \sum_{t = 1}^T l_t[\bx_0[t]] + \log[K \sfq[\bx_0] ]
\leq \frac{\eta T}{8} + \frac{1}{\eta} \log \Big[ K^\eta \sum_{\bx} \sfq[\bx]^\eta \Big].
\end{align*}
Since the inequality holds for any sequence $\bx_0 \in \Sigma^T$, it
implies the following upper bound on the weighted regret:
\begin{equation*}
  \Reg_T 
  \leq \frac{\eta T}{8} + \frac{1}{\eta} \log \Big[ K^\eta \sum_{\bx} \sfq[\bx]^\eta \Big].
\end{equation*}
By Jensen's inequality, the inequality
$\frac{1}{K}\sum_{\bx} \sfq[\bx]^\eta \leq \Big(\frac{1}{K}\sum_{\bx}
\sfq[\bx] \Big)^\eta = \frac{1}{K^\eta}$ holds for $\eta \in (0,
1)$. This implies the following general upper bounds on the weighted
regret:
\begin{equation*}
  \Reg_T 
  \leq \frac{\eta T}{8} + \frac{1}{\eta} \log K.
\end{equation*}
The weighted regret can also be upper bounded in terms of
the R\'enyi entropy. Observe that
\begin{align*}
\frac{\eta T}{8} + \frac{1}{\eta} \log \Big[ K^\eta \sum_{\bx}
  \sfq[\bx]^\eta \Big]
  & = \frac{\eta T}{8} + \frac{1-\eta}{\eta} H_\eta(\sfq) + \log K.
\end{align*}
$\eta \mapsto H_\eta(\sfq)$ is known to be a non-increasing function
(see e.g. \citep{VanErvenHarremos2014}). It follows that
$\eta \mapsto \frac{\eta}{\sqrt{H_\eta}(\sfq)}$ is an increasing
function that increases at least linearly. If we assume that $\sfq$ is
supported on more than a single sequence, then, we have
$H_0(\sfq) > 0$. Thus, for any $T$, there exists a unique $\eta^*$
such that
$\frac{\eta^*}{\sqrt{H_{\eta^*}}(\sfq)} = \sqrt{\frac{8}{T}}$.
Furthermore, for $\eta \leq \eta^*$, the following inequality holds:
$\frac{\eta}{\sqrt{H_\eta(\sfq)}} \leq \sqrt{\frac{8}{T}}$. Thus, we
can write
\begin{align*}
  \frac{\eta T}{8} + \frac{1}{\eta} \log \Big[ K^\eta \sum_{\bx}
  \sfq[\bx]^\eta \Big]
  & \leq \inf_{\eta \leq \eta^*} \frac{\eta T}{8} + \frac{1}{\eta}
    H_\eta(\sfq) - H_\eta(\sfq) + \log K \\
  & \leq \sqrt{\frac{T H_{\eta^*}(\sfq)}{2}} - H_{\eta^*}(\sfq) + \log
    K.
\end{align*}
The upper bound on the unweighted regret is obtained straightforwardly
from the previous derivations using $\sfq[\bx] = \frac{1}{K}$.
\end{proof}
Note that when the losses are \emph{mixing}, we can also derive better 
constant-in-time regret guarantees by avoiding the use of Hoeffding's 
inequality.

\section{Proof of Theorem~\ref{th:WFAapprox}}
\label{app:WFAapprox}

\begin{reptheorem}{th:WFAapprox}
The weighted regret of the \AWM\ algorithm with respect to the WFA
$\sC$ when run with $\h \sC_T$ instead of $\sC_T$ can be
upper bounded as follows:
\begin{align*}
  \Reg_T(\cA, \sC) 
  & \leq \frac{\eta T}{8} + \frac{1}{\eta} \log \Big[K^\eta \sum_{\bx} \h
  \sfq[\bx]^\eta \Big] + D_\infty(\sfq \| \h \sfq) 
 \leq \frac{\eta T}{8} + \frac{1}{\eta} \log K + D_\infty(\sfq \| \h \sfq). 
\end{align*}
Its unweighted regret can be upper bounded as follows:
\begin{equation*}
\Reg_T^0(\cA, \sC) \leq \max_{\sC(\bx) > 0} \frac{\eta T}{8} + \frac{1}{\eta} \log \bigg[
    \frac{1}{\sfq[\bx]}\bigg] + \frac{1}{\eta}
D_\infty(\sfq \| \h \sfq).
\end{equation*}
\end{reptheorem}

\begin{proof}
  By Theorem~\ref{th:awm} (and its proof), for any sequence
  $\bx_0 \in \Sigma^T$, the following upper bound holds for the
  cumulative loss of \AWM\ run with $\h \sC_T$:
\begin{align*}
& \sum_{t = 1}^T \sfp_t \cdot \bfl_t - \sum_{t = 1}^T l_t[\bx_0[t]] +
  \log[\h \sfq[\bx_0] ]
\leq \frac{\eta T}{8} + \frac{1}{\eta} \log \Big[\sum_{\bx} \h
  \sfq[\bx]^\eta \Big].
\end{align*}
Thus, for any sequence $\bx_0 \in \Sigma^T$ accepted by $\sC_T$, we can write
\begin{equation*}
\sum_{t = 1}^T \sfp_t \cdot \bfl_t - \sum_{t = 1}^T l_t[\bx_0[t]] +
  \log[\sfq[\bx_0] K ]
\leq \frac{\eta T}{8} + \frac{1}{\eta} \log \Big[K^\eta \sum_{\bx} \h
  \sfq[\bx]^\eta \Big] + \log \bigg[\frac{\sfq[\bx_0]}{\h \sfq[\bx_0]} \bigg],
\end{equation*}
which implies the following upper bound on the weighted regret:
\begin{align*}
\Reg_T(\cA, \sC)
& \leq \frac{\eta T}{8} + \frac{1}{\eta} \log \Big[K^\eta \sum_{\bx} \h
  \sfq[\bx]^\eta \Big] + \sup_{\sC_T(\bx_0) > 0} \log
  \bigg[\frac{\sfq[\bx_0]}{\h \sfq[\bx_0]} \bigg]\\
& \leq \frac{\eta T}{8} + \frac{1}{\eta} \log \Big[K^\eta \sum_{\bx} \h
  \sfq[\bx]^\eta \Big] + D_\infty(\sfq \| \h \sfq) 
\end{align*}
As in the proof of Theorem~\ref{th:awm}, by Jensen's inequality,
$\log \Big[K^\eta \sum_{\bx} \h
  \sfq[\bx]^\eta \Big] \leq \log K$, which implies the second inequality.

Similarly, by the proof of Theorem~\ref{th:awm}, the unweighted regret
of \AWM\ run with $\h \sC_T$ can be upper bounded as follows:
\begin{align*}
  \sum_{t = 1}^T \sfp_t \cdot \bfl_t -  \sum_{t = 1}^T l_t[z_t]  
  & \leq  \max_{\sC(\bx) > 0} \frac{\eta T}{8} + \frac{1}{\eta} \log \Big[
    \frac{1}{\h \sfq[\bx]}\Big]
    = \max_{\sC(\bx) > 0} \frac{\eta T}{8} + \frac{1}{\eta} \log \bigg[
    \frac{1}{\sfq[\bx]}\bigg] + 
    \frac{1}{\eta} \log \bigg[\frac{\sfq[\bx]}{\h \sfq[\bx]} \bigg],
\end{align*}
which completes the proof.
\end{proof}

\newpage

\section{Failure transition algorithms}
\label{app:phiaut}


The computational complexity of the \AWM\ algorithm presented in
Section~\ref{sec:AWM} is based on the size of the composed
automaton $\sC \cap \sS_T$, which itself is related to the
original size of $\sC$. Similarly, if we were to apply \AWM\
to an $n$-gram approximation, the computational complexity of the algorithm
depends on the size of the approximating automaton. 
In this section, we introduce a technique to improve the computational cost
of \AWM\ by reducing the size of the automaton, using the notion
of \emph{failure transition (or $\phi$-transition)}.

$\phi$-transitions are special transitions characterized by the semantic of 
``other''. If, at a state $q$, there is no outgoing transition labeled with
$a \in \Sigma$ and there is a $\phi$-transition leaving $q$
and reaching $q'$, then 
the failure transition is taken instead without
consuming the label, and the next state is determined using the
transitions leaving $q'$. A $\phi$-automaton is an automaton with $\phi$-transitions.
We assume that there is no $\phi$-cycle in
any of our $\phi$-automata, and that there is at most one failure
transition leaving any state. This implies that the number of
consecutive failure transitions taken is bounded.

A failure transition can often replicate
the role of multiple standard transitions when there is ``symmetry''
within an automaton, that is when there are many transitions leading
to the same state from different states that consume the same set of
labels. Figure~\ref{fig:phiexample} illustrates such a case.

\begin{figure}[t]
  \centering
\begin{tabular}{c@{\hspace{1cm}}c}
\includegraphics[scale=0.4]{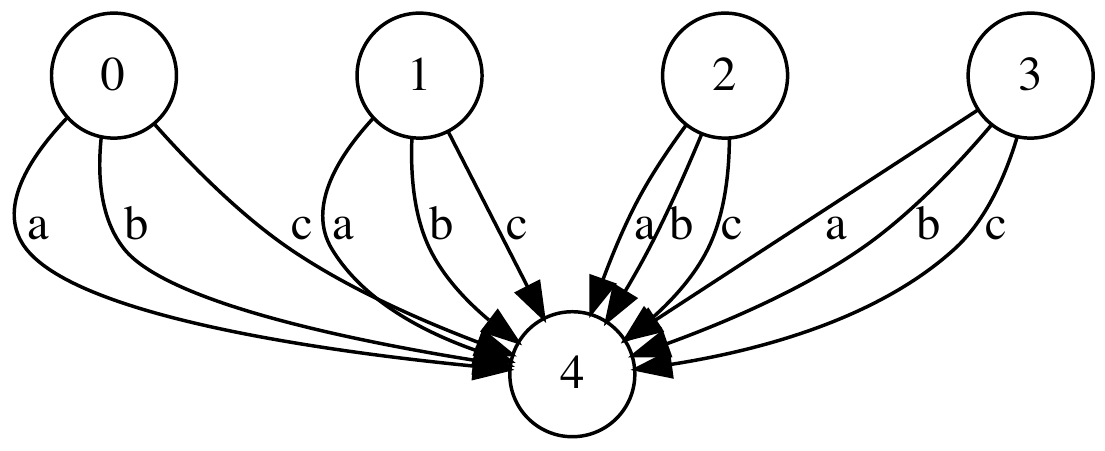} &
  \includegraphics[scale=0.4]{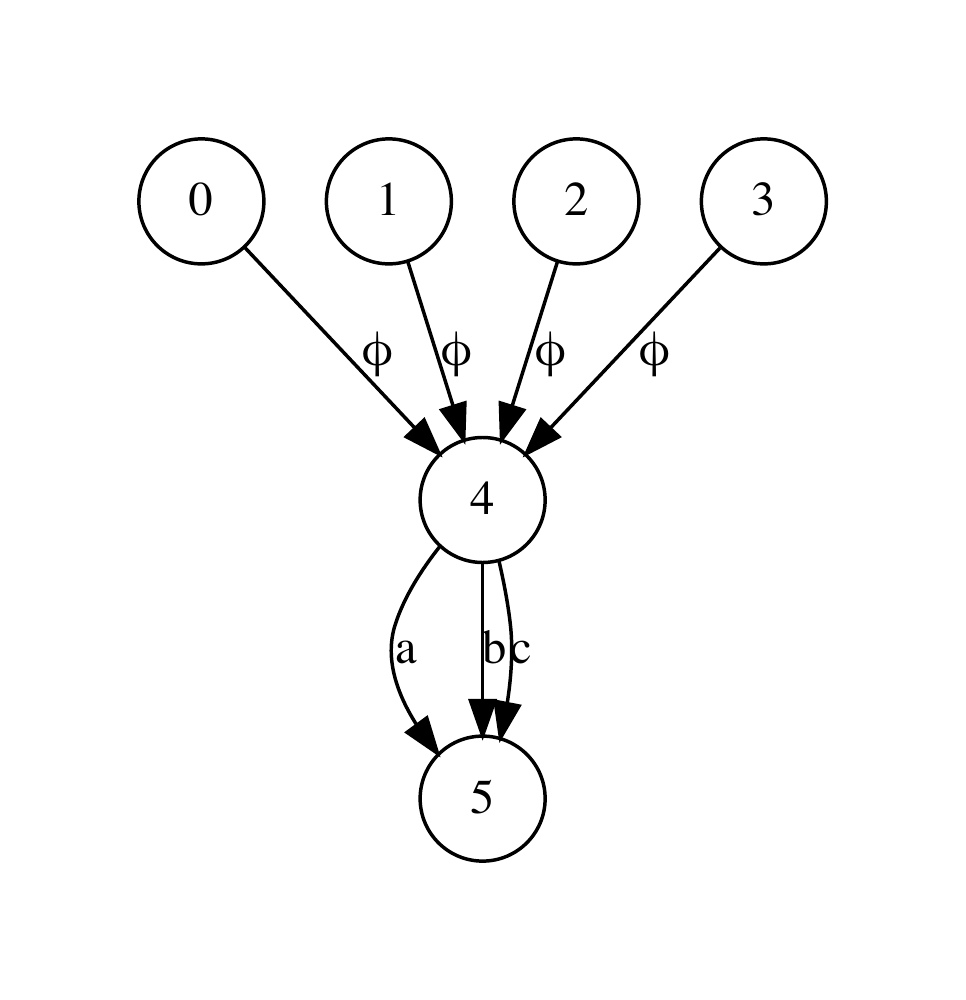} \\
(a) & (b)
\end{tabular}
  \caption{Example of the compression achieved by introducing a failure transition. (a) Standard automaton. (b) $\phi$-automaton.}  \label{fig:phiexample}
\end{figure}

\subsection{Conversion}

Notice that in Figure~\ref{fig:phiexample}, the introduction of a failure transition
removed $|S|$ transitions from $|Q|$ parent states while introducing
$|Q|$ $\phi$-transitions from each of the parent states to a new state
$q'$, and $|S|$ transitions from $q'$ to $q$. Thus, the change in the
number of transitions is $|S^*| + |Q^*| - |S^*| |Q^*|$. This fact can
be exploited to design an algorithm that iterates through the states
of an automaton, and for each state, determines whether it is
beneficial to introduce a failure transition between that state and (a
subset of) its parents. We call this algorithm, $\phi$-Convert, which
uses another algorithm, \textsc{$\phi$-SourceSubset} as a subroutine
to greedily select a candidate set of parent states from which to
introduce a $\phi$-transition for each state. The pseudocode for \textsc{$\phi$-Convert}
and \textsc{$\phi$-SourceSubset} are presented in Algorithm~\ref{alg:phiconvert}
and Algorithm~\ref{alg:ss} respectively. 

\begin{algorithm2e}[t]
  \TitleOfAlgo{\textsc{$\phi$-Convert}($\sC$)}
  \ForEach{\upshape \mbox{\textbf{non-initial state} } $q \in \sC$}{
    $S^*, Q^* \gets $ \textsc{$\phi$-SourceSubset}($\sC$, $q$) \\
    \If{$|S^*| + |Q^*| < |S^*| |Q^*|$}{
      $\tilde{q} \gets \textsc{NewState}(\sC)$ \\
      $E_\sC \gets E_\sC \cup \{(q, \phi, 1, \tilde{q})\}$\\ 
      \ForEach{$q' \in Q^*$}{
        \ForEach{$e' \in E_\sC[q']$}{
          \If{$(\lab[e'], \weight[e']) \in S^*$}{
            $E_\sC \gets E_\sC \cup \{(\tilde{q}, \lab[e'], \weight[e'], q)\}$\\
            $\textsc{DeleteTransition}[E_\sC, e']$
          }
        }
      }
    }
  }
\caption{\textsc{$\phi$-Convert}.} 
\label{alg:phiconvert}
\end{algorithm2e}

\begin{algorithm2e}[t]
  \TitleOfAlgo{\textsc{$\phi$-SourceSubset}($\sC, q$)}
  $(S_0, Q_0) \gets (\emptyset, \emptyset)$ \\
  $k^* \gets 1$ \\
  \For{$k \gets 1$ \upshape \mbox{\textbf{to}} $|$ \upshape \mbox{Parents} $[q]|$}{
    $q_k \gets \argmax_{q' \in \text{Parents}[q] \setminus Q_{k-1}} \left|(a,w) \in \Sigma \times \Rset_+ \colon \forall \tilde{q} \in \text{Parents}[q] \cup \{q'\}, (\tilde{q}, a, w, q) \in E_\sC\right|$ \\ 
    $S_k \gets \left|(a,w) \in \Sigma \times \Rset_+ \colon \forall \tilde{q} \in \text{Parents}[q] \cup \{q_k\}, (\tilde{q}, a, w, q) \in E_\sC\right|$ \\
    $Q_k \gets Q_{k-1} \cup \{q_k\}$ \\
    $k^* \gets \argmax_{j \in \{k,k^*\}} \{|S_j| |Q_j| - (|S_j| + |Q_j|)\}$
  }
  \RETURN{($S_{k^*}$, $Q_{k^*}$)} 
\caption{\textsc{$\phi$-SourceSubset}.} 
\label{alg:ss}
\end{algorithm2e}

Recall that the two main automata operations required for \AWM\ are
intersection and shortest-distance. While these two operations 
are standard for weighted automata, it is not as clear how one can
perform them over weighted $\phi$-automata. We now extend both to
$\phi$-automata.

\subsection{Intersection using a $\phi$-filter}

One of the main automata operations required for \AWM\ is
intersection.  The standard algorithm for intersection of automata
(Appendix~\ref{app:intersection}), which is based on matching
transitions, can return an incorrect result in the presence of
$\phi$-transitions.  Specifically, the algorithm may produce multiple
$\phi$-paths between two states, which leads to redundancy and 
incorrect weights.



Redundant $\phi$-paths are generated by standard intersection
algorithms because when the algorithm is in state $q_1$ in WFA $\sC_1$ and
state $q_2$ in $\sC_2$, both of which contain outgoing
$\phi$-transitions, the algorithm may take any of the following steps:
(1) move forward on a $\phi$-transition in $\sC_1$ while staying at
$q_2$; (2) move forward on a $\phi$-transition in $\sC_2$
while staying in $\sC_1$; or (3) move forward in both $\sC_1$ and
$\sC_2$.

To avoid this situation, we introduce the concept of a \emph{$\phi$-filter},
which is a \emph{finite state transducer (FST)} that can filter out all but 
one $\phi$-path between
any two states. 

Our $\phi$-filter is designed to modify the two input automata in a way
that will distinguish between the above cases. In
$\sC_1$, for every $\phi$-transition, we rename the label $\phi$ as
$\phi_2$. Moreover, at the source and destination states of every
$\phi$-transition, we introduce new self-loop transitions labeled with
$\phi_1$ and with weight $1$.  Thus, a transition labeled with
$\phi_2$ will indicate a ``move forward,'' while a transition labeled
with $\phi_1$ will indicate a ``stay.''  Similarly, in $\sC_2$, we
rename the $\phi$ labels as $\phi_1$, and we introduce self-loops
labeled with $\phi_2$ and weight $1$ at the source and destination
states of every $\phi$-transition. With these modifications, any
$\phi$-path resulting from the composition algorithm will include
transitions of the form: (1) $(\phi_2 : \phi_2)$; (2)
$(\phi_1 : \phi_1)$; or (3) $(\phi_2 : \phi_1)$. 
\begin{figure}[t]
  \centering
  \includegraphics[scale=0.5]{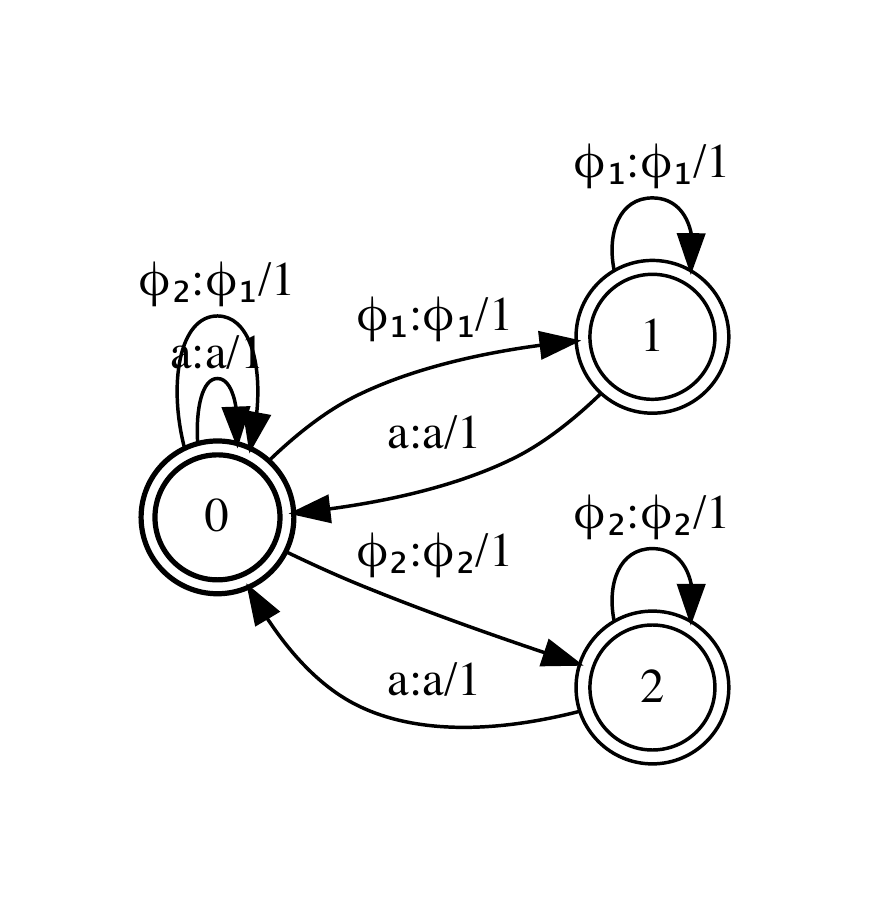} 
  \caption{Illustration of the $\phi$-filter $\sF$.}  
  \label{fig:filter}
\end{figure}

Now consider the finite-state transducer $\sF$ illustrated in
Figure~\ref{fig:filter}, which will serve as our $\phi$-filter.  The
composition of any two $\phi$-automata and the $\phi$-filter $\sF$,
$\sC_1 \circ \sF \circ \sC_2$, will result in a finite-state
transducer whose transitions have labels in
$\{(a:a)\}_{a \in \Sigma} \cup \{(\phi_2:\phi_2), (\phi_1: \phi_1),
(\phi_2:\phi_1)\}$.\footnote{Composition is a standard algorithm for
  weighted finite-state transducers which coincides with the
  intersection operation in the special case of WFA (see \cite{Mohri2009}).}
Moreover, we identify all label pairs in
$\{(\phi_2:\phi_2), (\phi_1: \phi_1), (\phi_2:\phi_1)\}$ using the
same semantic of ``other'' as we did with $\phi$. Thus, we can
identify all label pairs in
$\{(\phi_2:\phi_2), (\phi_1: \phi_1), (\phi_2:\phi_1)\}$ with the
single pair $(\phi:\phi)$ and treat the result of composition as
simply a weighted finite automaton.

\subsection{Update of $\balpha$ using a modified shortest-distance algorithm}

The other key ingredient of the \AWM\ algorithm is the update of $\balpha$
using the shortest-distance algorithm for WFA.
However, updating $\balpha$ as we did in \AWM\ may result in summing
over  `obsolete $\phi$-transitions'.  For
example, if at a given state $q$, there is a transition labeled with
$a$ to $q'$ and a $\phi$-transition whose destination state has a
single outgoing transition also labeled with $a$ to $q'$, the second
path should not be considered.

To account for these types of situations, we use the fact that the
semiring $(\Rset_+, +, \times, 0, 1)$ admits a natural extension to a
ring structure under the standard additive inverse $-1$. Specifically,
upon encountering a transition $e$ labeled with $a$ leaving state $q$,
we will check for $\phi$-transitions with destination states that
admit further transitions $e'$ labeled with $a$.  Any such transition
should not be considered under the semantic of the $\phi$-transition
and thus should not contribute any weight to the distance to
$\alpha[\dest[e']]$.  To correctly account for these paths, we will
\emph{preemptively} subtract the weight of $e'$ from its destination
state. When the algorithm processes the $\phi$-transition directly, it
will add this weight back so that the total contribution of this path
is zero.

\subsection{\textsc{$\phi$-AWM} algorithm}

With the addition of the $\phi$-filter and the modified $\balpha$ update described
above, we can present \textsc{$\phi$-AWM}, an extension of \AWM\ that can handle
$\phi$-automata. Given an input automaton (not necessarily with
$\phi$-transitions), the algorithm first calls \textsc{$\phi$-Convert}
to determine whether it is beneficial to introduce
$\phi$-transitions. The algorithm then composes the output with
$\Sigma^T$ (using the $\phi$-filter) to compute the set of sequences of
length $T$ that are accepted by $\sC$. Then, the algorithm updates the
weights of the automaton in a similar manner as in \AWM\, with the
additional adjustment of preemptively accounting for
$\phi$-transitions.  Algorithm~\ref{alg:phiawm} presents the
pseudocode for \textsc{$\phi$-AWM}.

\begin{algorithm2e}[t]
  \TitleOfAlgo{\textsc{$\phi$-AWM}($\sC$, $\eta$)}
  $\sC \gets \textsc{$\phi$-Convert}(\sC)$ \\ 
  $\sB \gets \sC \cap \sF \cap \sS_T$ \\
  $\sA \gets \textsc{Weight-Pushing}(\sB^\eta)$ \\
  $\bbeta \gets \textsc{BwdDist}(\sA)$ \\
  $\balpha \gets 0$; $\balpha[I_\sA] \gets 1$ \\ 
  \ForEach{$e \in E_{\sA}^{0 \to 1}$}{
      $\sfp_1[\lab[e]] \gets \weight[e]$.
    }
 \For{$t \gets 1$ \KwTo $T$}{
    $i_t \gets $\textsc{Sample}($\sfp_t$); \textsc{Play}($i_t$); \textsc{Receive}($\bfl_t$)\\
    $Z \gets 0$; $\bw \gets 0$\\ 
    \ForEach{$e \in E_{\sA}^{t \to t + 1}$}{
      $\weight[e] \gets \weight[e] \, e^{-\eta l_t[\lab[e]]}$ \\
      $\bw[\lab[e]] \gets \bw[\lab[e]] + \balpha[\src[e]] \, \weight[e] \, \bbeta[\dest[e]]$\\
      $Z \gets Z + \bw[\lab[e]]$\\
      $\balpha[\dest[e]] \gets \balpha[\dest[e]] + \balpha[\src[e]] \,
      \weight[e] $\\
      \If{$\lab[e] \neq \phi$}{
        $\tilde{q} = \src[e]$; $w_\phi \gets 1$ \\
        \While{$\exists e_\phi \in E[\tilde{q}]$ \WITH $\lab[e_\phi] = \phi$}{
          $w_\phi \gets w_\phi\, \weight[e_\phi]$ \\
          \eIf{$\exists e' \in E[\dest[e_\phi]]$ \WITH $\lab[e'] = \lab[e]$}{
            $\alpha[\dest[e']] \gets \alpha[\dest[e']] - w_\phi \weight[e']$ \\
            \textsc{Break} \\
          }{
            $\tilde{q}\gets \dest[e_\phi]$\\
          }
        }
      }
    }
    $\sfp_{t + 1} \gets \frac{\bw}{Z}$ \\ 
  }
\caption{\textsc{$\phi$-AutomataWeightedMajority}($\phi$-AWM).} 
\label{alg:phiawm}
\end{algorithm2e}

Since the update of $\sfp_t$ in \textsc{$\phi$-AWM} is mathematically
equivalent to the one in \AWM\, we obtain the same regret
guarantees as in Theorem~\ref{th:awm}. Moreover, if we denote by
$N_\phi(Q_{\sC_T})$ the maximum number of consecutive
$\phi$-transitions leaving states in $Q_{\sC_T}$, the total
computational cost of the algorithm is in
$\cO\left(\sum_{t = 1}^T N_\phi(Q_{\sC_T, t-1}) |E_{\sA}^{t \to t + 1}| \right)$.


For the $k$-shifting automaton, the per-iteration computational
complexity of \textsc{$\phi$-AWM} is now $\cO(Nk)$, since there is at
most one consecutive $\phi$-transition in the output of
$\phi$-Convert, and we now aggregate transitions at each time using
failure transitions. This is a factor of $N$ better than that of 
\textsc{AWM}, and only a factor of $k$ worse than the \textsc{Fixed-Share}
algorithm of \cite{HerbsterWarmuth1998}. If we intersect the
$k$-shifting automaton with $\Sigma^T$, approximate the result with a bigram model,
and then convert this model into a $\phi$-automaton, we obtain
an algorithm that runs in $\cO(N)$, which is the same as that of 
\textsc{Fixed-Share}. See Figure~\ref{fig:bigramkshift} for an illustration. 

\begin{figure}[t]
\centering
\includegraphics[scale=0.55]{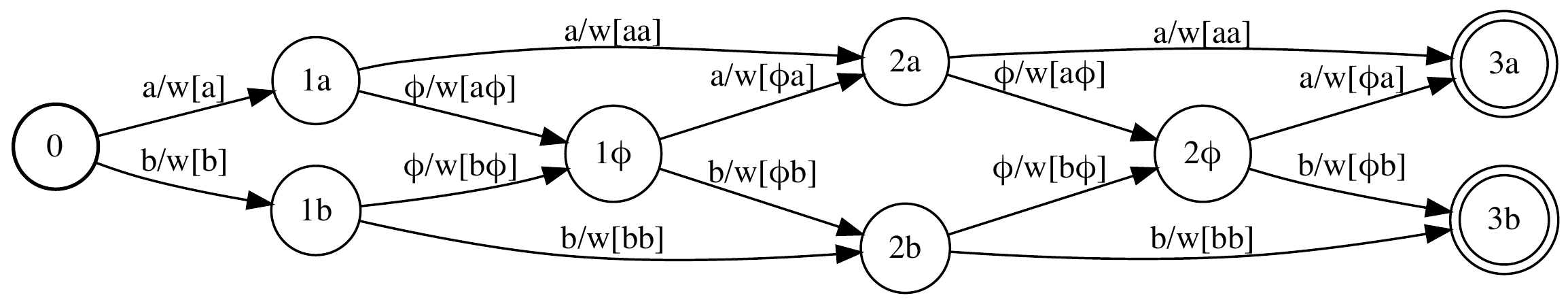}
\caption{An illustration of a bigram model
  approximating the $k$-shifting automaton composed with $\sS$. 
  \textsc{$\phi$-Convert} has been applied to the bigram model, making it
  smaller than a standard bigram model.}
\label{fig:bigramkshift}
\end{figure}

\newpage

\section{\textsc{Prod-EG}}
\label{app:Prod-EG}

The pseudocode of the \textsc{Prod-EG} algorithm, which is based on a simple
multiplicative update, is given in Algorithm~\ref{alg:prodeg}. The
following provides a general guarantee for the convergence of the
algorithm.

\begin{algorithm2e}[t]
  \TitleOfAlgo{\textsc{Prod-EG}($\sfq_1 \in (\Delta_N)^m$, $\eta$)}
  \For{$s = 1,2,\ldots, \tau$}{
    \textsc{Play}($\sfq_s$) \\
    \textsc{Receive}($\nabla f(\sfq_s)$)\\
    \For{$j = 1,2,\ldots, m$}{
      \For{$i = 1,2,\ldots, N$}{
        $\sfq_{s+1, j}(i) = \sfq_{s, j} e^{-\eta \frac{\partial f}{\partial \sfq_{j}(i)}(\sfq_{s, j})}$
       }
    }
  }
  \caption{\textsc{Prod-EG}.} 
\label{alg:prodeg}
\end{algorithm2e}

\begin{theorem}[\textsc{Product-Exponentiated Gradient (Prod-EG)}]
  \label{th:prodeg}
  \text{ } 

  Let $(\Delta_N)^m$ be the product of $m$ $(N-1)$-dimensional simplices, and let\\
  $f\colon (\Delta_N)^m \to \Rset$ be a convex function whose partial
  subgradients have absolute values all bounded by $L$. Let
  $\sfq_{1,j}(i) = \frac{1}{N}$ for $i \in [N]$ and $j \in [m]$.  Then,
  \textsc{Prod-EG} benefits from the following guarantee:
\begin{equation*}
f\bigg(\frac{1}{\tau} \sum_{s = 1}^\tau \sfq_s\bigg) - f(\sfq^*) \leq
\frac{1}{\eta \tau } m \log(N) + 2 \eta L.
\end{equation*}
\end{theorem}
\begin{proof}
  Consider the mirror map $\psi\colon (\Delta_N)^m \to \Rset$ defined by $\psi(\sfq) = \sum_{j=1}^m \sum_{i=1}^N \sfq_{j}(i) \log \sfq_j(i) $.
  This induces the Bregman divergence:
\begin{equation*} 
B_\psi(\sfq, \sfq') = \sum_{j=1}^m \sum_{i=1}^N \sfq_j(i) \log
\left(\frac{ \sfq_j(i)}{ \sfq'_j(i)} \right).
\end{equation*}
Since each relative entropy is $1$-strongly convex with respect to the
$l_1$ norm over a single simplex, the additivity of strong convexity
implies that $B_\psi$ is $1$-strongly convex with respect to the $l_1$
norm defined over$(\Delta_N)^m$.

The update described in the theorem statement corresponds to the
mirror descent update based on $B_\psi$:
\begin{equation*}
\sfq_{s+1} = \argmin_{\sfq \in (\Delta_N)^m} \langle g_s , \sfq
\rangle + B_\psi(\sfq, \sfq_s).
\end{equation*}
where $g_s \in \partial(f(\sfq_s))$ is an element of the subgradient
of $f$ at $\sfq_s$.  Thus, the standard mirror descent regret bound
(e.g. \citep{Bubeck2015}) implies that
\begin{equation*}
\frac{1}{\tau} \sum_{s = 1}^\tau f(\sfq_s) - f(\sfq^*) \leq \frac{1}{\eta \tau } B_\psi(\sfq^*, \sfq_1) + \eta 2 L.
\end{equation*}
The result now follows from the fact the observation that 
$B_\psi(\sfq^*, \sfq_1) \leq m \log(N)$.
\end{proof}

For the minimum R\'{e}nyi divergence optimization problem
\eqref{eq:minrenyingram}, we can apply \textsc{Prod-EG} to
the product of $m = \sum_{j=1}^n |\Sigma|^{n-j}$ simplices, each one
corresponding to a conditional probability with a specific history.
First, we remark that the subgradient of the maximum of a family of
convex functions at a point can always be chosen from the subgradient
of the maximizing function at that point.  Specifically, let
$\{f_\alpha\}_{\alpha \in \cA}$ be a family of convex functions, and
let $\alpha(x) = \argmax_{\alpha} f_\alpha(x)$. Then, it follows that
\begin{align*}
  \max_\alpha f_\alpha(x) - \max_\alpha f_\alpha(y)
  &\geq f_{\alpha(y)}(x) - f_{\alpha(y)}(y) 
  \geq \langle \nabla f_{\alpha(y)}(y), x - y \rangle .
\end{align*}

Let $\bx$ be the maximizing path of the minimum R\'{e}nyi divergence
objective. 
\ignore{We will use the $\vee$ symbol to denote the maximum between
two values, and the $\wedge$ symbol the denote the minimum.} 
We can then write
\begin{align*}
  \log \bigg[ \frac{\sfq[\bx]}{\h \sfq_\sfw[\bx]} \bigg]
  & = \sfq[\bx] - \sum_{t = 1}^T \log \sfw \big[\bx[t] \big| \bx_{\max(t - n +
  1, 1)}^{t - 1} \big] \\
  & = \sfq[\bx] - \sum_{j = 1}^n \sum_{\bz_1^j \in \Sigma^j} \sum_{t = 1}^T 1_{j = \min(t, n)} 1_{\bx_{\max(t - n
    + 1, 1)}^t = \bz_1^j} \log \sfw \Big[\bz[j] \Big| \bz_1^{j - 1} \Big].
\end{align*}
Thus, its partial derivative with respect to $\sfw \Big[\bz[j] \Big| \bz_1^{j - 1} \Big]$ is:
\begin{align*}
  \frac{\partial}{\partial \sfw \Big[\bz[j] \Big| \bz_1^{j - 1} \Big]}
  \log \bigg[ \frac{\sfq[\bx]}{\h \sfq_\sfw[\bx]} \bigg]
  & = - \sum_{t = 1}^T
    \frac{1_{j = \min(t, n)} 1_{\bx_{\max(t - n
    + 1, 1)}^t = \bz_1^j}}{\sfw \Big[\bz[j] \Big| \bz_1^{j - 1} \Big]}.
\end{align*}
Thus, by tuning \textsc{Prod-EG} with an adaptive learning rate
$$\eta_t \propto \frac{1}{\sqrt{ \sum_{s = 1}^t \left \|\nabla  \log
      \Big[ \frac{\sfq[\bx(s)]}{\h \sfq_{\sfw_s}[\bx(s)]} \Big]
    \right\|_\infty^2}},$$ where
$\bx(s) = \argmax_{\bx \in \sC_T} \log \Big[ \frac{\sfq[\bx]}{\h
  \sfq_{\sfw_s}[\bx]} \Big]$, we can derive the following guarantee
for \textsc{Prod-EG} applied to the $n$-gram approximation problem.
\begin{corollary}[$n$-gram approximation guarantee]
\label{cor:prodegngram}
There exists an optimization algorithm outputting a sequence of
conditional probabilities $(\sfq_t)_{t = 1}^\infty$ such that
$\left(\frac{1}{T} \sum_{t = 1}^T \sfq_t\right)$ approximates the
$\infty$-R\'{e}nyi optimal $n$-gram solution with the following
guarantee:
  \begin{align*}
    &F\Bigg(\frac{1}{\tau} \sum_{s = 1}^\tau \sfq_t\Bigg) - F(\sfq^*) \leq
  \sqrt{\frac{2 N^{n} \log(N) \sum_{s = 1}^\tau \max_{\stackrel{j \in
      [n]}{\bz_1^j \in \Sigma^j}} \bigg| 
      \sum_{t=1}^T \frac{1_{j = \min(t, n)} 1_{\bx_{\max(t - n
    + 1, 1)}^t = \bz_1^j}}{\sfw_s \big[\bz[j] \big| \bz_1^{j - 1} \big]}
      \bigg|}{(N - 1) T^2}}.
  \end{align*}
\end{corollary}
Each iteration of \textsc{Prod-EG} admits a computational complexity
that is linear in the dimension of the feature space. Since we have
specified an $n$-gram model as the product of $\frac{N^n - 1}{N - 1}$
simplices, the total per-iteration cost of solving the convex
optimization problem is in
$\cO\left(\frac{N(N^n - 1)}{N-1} \right) = \cO(N^n)$. Since the
minimum R\'{e}nyi divergence is not Lipschitz, the maximizing ratio in
the convergence guarantee may also become large when the choice of $n$
is too small. In all cases, observe that this approximation problem
can be solved offline.

\newpage

\section{Minimum R\'enyi divergence unigram models}
\label{app:unigram}

\ignore{
If we restrict ourselves to $\sC_T$ with uniform weights and
$|\Sigma| = 2$, then we can also provide an explicit solution for
unigram automata. The solution is obtained from the paths with the
smallest number of occurrences of each symbol, which can be
straightforwardly found via a shortest-path algorithm in linear time.
}

\begin{reptheorem}{th:inftyrd1gram}
Assume that $\sC_T$ admits uniform weights over all paths and
$\Sigma = \set{a_1, a_2}$.  For $j \in \set{1, 2}$, let $n(a_j)$ be the smallest
number of occurrences of $a_j$ in a path of $\sC_T$.  For any $j \in
\set{1, 2}$, define
\begin{equation*}
\sfq[a_j] = \frac{\max\left\{1, \frac{n(a_j)}{T -
      n(a_j)}\right\}}{1+\max\left\{1, \frac{n(a_j)}{T -
      n(a_j)}\right\}}.
\end{equation*}

Then, the unigram model $\sfw \in \cW_1$ solution of
$\infty$-R\'{e}nyi divergence optimization problem is defined by
$\sfw[a_{j^*}] = \sfq[a_{j*}]$, $\sfw[a_{j'}] = 1 - \sfw[a_{j^*}]$,
with
$j^* = \argmax_{j \in \set{1, 2}} \ n(a_j) \log \sfq[a_j] + \left[T -
  n(a_j) \right] \log\big[ 1 - \sfq[a_j] \big]$.
\end{reptheorem}
\begin{proof}
  We seek a unigram distribution $\sfq_{\sfw}$ that is a solution of:
  $$\min_{\sfw \in \cW_1} \sup_{\bx \in \sC_T} \log\left( \frac{ \sfq[\bx]}{\sfq_{\sfw}[\bx]} \right).$$
  Since $\sC_T$ admits uniform weights, $\sfq[\bx]  = \frac{1}{|\sC_T|}$, and since $\sfq_{\sfw}$ is the distribution induced by a unigram model,
  $\log \sfq_{\sfw}[\bx]$ can be expressed as follows:
  $$\log \sfq_{\sfw}[\bx] = n_\bx(a_1) \log p(a_1) + \left[T - n_\bx(a_1) \right] \log\left( 1 - p(a_1) \right),$$
  where $p(a_j)$ is the automaton's weight on transitions labeled with $a_j$ and $n_\bx(a_j)$ is the count of $a_j$ in the sequence $\bx$. Thus, the optimization problem is equivalent to the following problem:
  $$ - \max_{p(a_1) \in [0,1]} \min_{\bx \in \sC_T}  n_\bx(a_1) \log p(a_1) + \left[T - n_\bx(a_1) \right] \log\left( 1 - p(a_1) \right).$$
  Denote the objective by $F(p(a_1), n_\bx(a_1))$. Then, the partial derivatives with respect to the label counts are given by
  \begin{align*}
    \frac{\partial F}{\partial n_\bx(a_1)} 
    & = \log p(a_1) - \log\Big( 1 - p(a_1) \Big).
  \end{align*}
  Thus, $\frac{\partial F}{\partial n_\bx(a_1)} \geq 0$ if and only if $p(a_1) \geq 1 - p(a_1)$. Furthermore, if $p(a_1) \geq 1 - p(a_1)$,
  then the sequence $\bx$ chosen in the optimization problem is the sequence with the minimal count of symbol $a_1$.
  Similarly, if $p(a_2) \geq 1 - p(a_2)$, then the sequence $\bx$ chosen in the optimization
  problem is the one with minimal count of $a_2$.

  Since we have either $p(a_1) \geq p(a_2)$ or vice versa (potentially both), we can write the optimization problem as:
  \begin{align*}
    -\max_{k \in \{1,2\}} \max_{\stackrel{p(a_j) \geq 1 - p(a_j)}{j\neq k}} \min_{\{n_\bx(a_j)\}_{j\neq k}\colon \bx \in \sC_T}  n_\bx(a_j) \log p(a_j) 
     + \left[T - n_\bx(a_j) \right] \log\left( 1 - p(a_j) \right).
  \end{align*}
  Given $k \in \{1,2\}$, let $\bx(k)$ be the sequence that minimizes $n_\bx(a_j)$ over all $\bx$ for $j \neq k$.
  Denote these counts $n_{\bx(k)}(a_j)$ by $n(a_j)$. Then we can rewrite the objective as:
  \begin{align*}
    -\max_{k=1,2,\ldots, N} \max_{\stackrel{p(a_j) \geq 1 - p(a_j)}{j\neq k}} n(a_j) \log p(a_j) 
     + \left[T - n(a_j) \right] \log\left( 1 - p(a_j) \right).
  \end{align*}
  Denote the objective for this new term by $\tilde{F}_k$, which is a function of $p(a_j)$. The partial derivative of $\tilde{F}_k$ with respect to $p(a_j)$ is:
  \begin{align*}
    \frac{\partial \tilde{F}_k}{\partial p(a_j)}
    & = \frac{n(a_j)}{p(a_j)} - \frac{T - n(a_j)}{1 - p(a_j)}, 
  \end{align*}
  which is equal to $0$ if and only if 
  \begin{align*}
    p(a_j) 
    & = \frac{n(a_j)}{T - n(a_j)} \left( 1 - p(a_j) \right) 
    = \max\left\{1, \frac{n(a_j)}{T - n(a_j)}\right\} \left( 1 - p(a_j) \right). 
  \end{align*}
  The last equality follows from our assumption that $p(a_j) \geq 1 - p(a_j)$. Now, let $\sfq(a_j)$ denote the probabilities that we have just computed. 
  Then, we can write the optimization problem of $\tilde{F}_k$ as:
  \begin{align*}
    -\max_{k \in \{1,2\}, j \in \{1,2\} \setminus \{k\}} n(a_j) \log \sfq(a_j) 
    + \left[T - n(a_j) \right] \log\left( 1 - \sfq(a_j) \right).
  \end{align*}
\end{proof}
\ignore{
Theorem~\ref{th:inftyrd1gram} shows that the solutions of the
$\infty$-R\'enyi divergence optimization are based on the $n$-gram
counts of sequences in $\sC_T$ with ``high entropy''. This can be
very different from the maximum likelihood solutions, which are based
on the average $n$-gram counts.  For instance, suppose we are under
the assumptions of Theorem~\ref{th:inftyrd1gram}, and specifically,
assume that there are $T$ sequences in $\sC_T$. Assume that one of
the sequences has $\left(\frac{1}{2} + \gamma\right)T$ occurrences of
$a_1$ for some small $\gamma > 0$ and that the other $T-1$ sequences
have $T-1$ occurrences of $a_1$. Then,
$n(a_1) = \left(\frac{1}{2} + \gamma\right)T$, and the solution
of the $\infty$-R\'enyi divergence optimization problem is given by
$\sfq_\infty(a_1) = \frac{1 + 2\gamma}{2}$ and
$\sfq_\infty(a_2) = \frac{1 - 2\gamma}{2}$. On the other hand,
the maximum-likelihood solution would be
$\sfq_1(a_1) = 1 + \frac{\gamma}{T} - \frac{3}{2T} + \frac{1}{T^2} \approx 1$
and $\sfq_1(a_2) = \frac{3}{2T} - \frac{\gamma}{T} - \frac{1}{T^2} \approx 0$
for large $T$.
}

\newpage
\section{Maximum likelihood $n$-gram models}
\label{app:mlapprox}

\begin{reptheorem}{th:bigramkshift}
Let $\sC_T$ be the $k$-shifting automaton for some $k$.  Then, the
bigram model $\sfw_2$ obtained by minimizing relative entropy is
defined for all $a_1, a_2 \in \Sigma$ by
\begin{equation*}
  \sfq_{\sfw_2}[a_1a_2]  = \frac{1}{N}\Big[ 1 - \frac{k}{(T - 1)}\Big] \, 1_{a_1
    = a_2} + \frac{1}{N} \Big[ \frac{k}{(T - 1)(N - 1)} \Big] \, 1_{a_1 \neq a_2}.
\end{equation*}
Moreover, its approximation error can be bounded by a constant
(independent of $T$):
\begin{equation*}
    D_\infty(\sfq \| \sfq_{\sfw_2}) \leq  - \log \big[ 1 - 2e^{-\frac{1}{12k}} \big].
\end{equation*}
\end{reptheorem}
\begin{proof}
  Let $a_1, a_2 \in \Sigma$. Then, we can write
\begin{align*}
     \sfq_{\sfw_2}[a_2 | a_1] 
     & = \sfq_{\sfw_2}[a_2 | a_1, a_2 = a_1]\, \sfq_{\sfw_2}[a_2 = a_1] + \sfq_{\sfw_2}[a_2 | a_1, a_2 \neq a_1]\, \sfq_{\sfw_2}[a_2 \neq a_1].
\end{align*}
Consider first the case where $a_2 = a_1$. Then,
$\sfq_{\sfw_2}[a_2 | a_1, a_2 = a_1] = 1$, and
$\sfq_{\sfw_2}[a_2 = a_1]$ is the expected number of times that we see
label $a_2$ agreeing with label $a_1$.  Since $\sfq$ is uniform for
the $k$-shifting automaton, the expected counts are pure counts, and
the probability that we see two consecutive labels agreeing is
$1 - \frac{k}{T-1}$.  Now, consider the case where $a_2 \neq a_1$. By
symmetry, $\sfq_{\sfw_2}[a_2 | a_1, a_2 \neq a_1] = \frac{1}{N-1}$,
since $a_2$ is equally likely to be any of the other $N-1$
labels. Moreover, $\sfq_{\sfw_2}[a_2 \neq a_1] = \frac{k}{T-1}$.
Thus, the following holds:
\begin{align*}
    \sfq_{\sfw_2}[a_2 | a_1]
    & = \frac{1}{N-1}\frac{k}{T-1} 1_{a_1 \neq a_2}  +  \left[ 1 - \frac{k}{T-1} \right] 1_{a_1 = a_2}. 
\end{align*}
By symmetry, we can write $\sfq_{\sfw_2}[a_1] = \frac{1}{N}$,
therefore,
\begin{equation*}
\sfq_{\sfw_2}[a_1a_2] = \sfq_{\sfw_2}[a_2| a_1] \sfq_{\sfw_2}[a_1] = \frac{k}{N(N-1)(T-1)} 1_{a_1 \neq a_2}  +  \left[ \frac{T-1-k}{N(T-1)} \right] 1_{a_1 = a_2}.
\end{equation*}
Since the $k$-shifting automaton has uniform weights and
$\sfq_{\sfw_2}$ is uniform on $\sC_T$, we can write for any string $\bx$
accepted by $\sC_T$:
  \begin{align*}
    \log\bigg[ \frac{\sfq[\bx]}{\sfq_{\sfw_2}[\bx]}\bigg]
    &=  \log\left[ \frac{1}{\sfq_{\sfw_2}[\bx]}\right] - \log(|\sC_T|) \\
    & = \log\left[ \frac{1}{\sfq_{\sfw_2}[\bz=  \bx | \bz \in \sC_T]
      \sfq_{\sfw_2}[\bz \in \sC_T] + \sfq_{\sfw_2}[\bz = \bx | \bz \notin
      \sC_T] \sfq_{\sfw_2}[\bz \notin \sC_T]}\right] - \log(|\sC_T|)\\
    & = \log\left[ \frac{1}{\frac{1}{|\sC_T|} \sfq_{\sfw_2}[\bz \in \sC_T] + \sfq_{\sfw_2}[\bz = \bx | \bz \notin \sC_T] \sfq_{\sfw_2}[\bz \notin \sC_T]}\right]  - \log(|\sC_T|)\\
    & \leq \log\left[ \frac{|\sC_T|}{\sfq_{\sfw_2}[\bz \in \sC_T]} \right]  - \log(|\sC_T|) 
    = \log\left[ \frac{1}{\sfq_{\sfw_2}[\bz \in \sC_T]} \right]. 
  \end{align*}
  The probability that a string $\bz$ is accepted by $\sC_T$ (under the
  distribution $\sfq_{\sA_2}$) is equal to the probability that it
  admits exactly $k$ shifts. Let $\xi_t = 1_{\{\text{$\bz$ shifts from
      $t-1$ to $t$}\}}$ be a random variable indicating whether there
  is a shift at the $t$-th symbol in sequence $\bz$. This is a Bernoulli
  random variable bounded by $1$ with mean $\frac{k}{T-1}$ and
  variance $\frac{k}{T-1} ( 1- \frac{k}{T-1})$.  Since each shift
  occurs with probability $\frac{k}{T-1}$, we can use Sanov's theorem
to write the following bound:
  \begin{align*}
    \sfq_{\sfw_2}[\bz \notin \sC_T]
    & = \sfq_{\sfw_2}\left[ \left| \sum_{t=2}^T \xi_t - k \right| > \frac{1}{2} \right] 
    \leq 2 e^{-(T-1) u}, 
  \end{align*}
  where
  $u = (T-1) \displaystyle\min\left\{ D\left(\frac{k +
        \frac{1}{2}}{T-1} \bigg \| \frac{k}{T-1}\right),
    D\left(\frac{k - \frac{1}{2}}{T-1} \bigg \| \frac{k}{T-1}\right)
  \right\}$.  We now give lower bounds on the relative entropy terms
  arguments of the minimum operator.  For the first term, using the
  inequalities $\log(1 + x) \geq \frac{x}{1 + \frac{x}{2}}$ and
  $\log(1 + x) < x$, we can write
\begin{align*}
    & - D\left(\frac{k + \frac{1}{2}}{T-1} \bigg\| \frac{k}{T-1}\right) \\ 
\ignore{
    & = - D\left(\frac{k}{T-1}\left(1 + \frac{1}{2k}\right) \bigg \| \frac{k}{T-1}\right) \\
    & = \bigg(1 + \frac{1}{2k} \bigg) \frac{k}{T-1} \log
      \frac{\frac{k}{T-1}}{\left(1 + \frac{1}{2k}
      \right)\frac{k}{T-1}} + \left(1 - \left(1 + \frac{1}{2k} \right)
      \frac{k}{T-1}\right) \log \frac{1 - \frac{k}{T-1}}{1 - \left(1 +
      \frac{1}{2k} \right) \frac{k}{T-1}} \\
}
    & = \left(1 + \frac{1}{2k} \right) \frac{k}{T-1} \log \frac{1}{1 + \frac{1}{2k} } + \left( 1 - \frac{k}{T-1} - \frac{1}{2k} \frac{k}{T-1}\right) \log\left(1 + \frac{\frac{1}{2k} \frac{k}{T-1}}{1 - \frac{k}{T-1} - \frac{1}{2k} \frac{k}{T-1}} \right) \\
    & \leq \left(1 + \frac{1}{2k} \right) \frac{k}{T-1} \frac{-\frac{1}{2k}}{1 + \frac{1}{4k}} + \left( 1 - \frac{k}{T-1} - \frac{1}{2k} \frac{k}{T-1} \right) \frac{\frac{1}{2k} \frac{k}{T-1}}{1 - \frac{k}{T-1} - \frac{1}{2k} \frac{k}{T-1}} \\
    & = \frac{1}{2k} \frac{k}{T-1} \left( 1 - \frac{ 1 + \frac{1}{2k} }{1 + \frac{1}{4k}} \right) 
    = \frac{ - \frac{1}{8k^2} \frac{k}{T-1}}{1 + \frac{1}{4k}} 
    = \frac{ - \frac{1}{4k^2} \frac{k}{T-1}}{2 + \frac{1}{4k}} 
    \leq  - \frac{1}{12 k(T-1)}.
  \end{align*}
  Similarly, we can write:
  \begin{align*}
    & -D \left( \left(1 - \frac{1}{2k} \right) \frac{k}{T-1} \bigg \| \frac{k}{T-1} \right) \\
\ignore{
    & = \left(1 - \frac{1}{2k} \right) \frac{k}{T-1} \log
      \frac{\frac{k}{T-1}}{\left(1 - \frac{1}{2k} \right)
      \frac{k}{T-1}} + \left(1 - \left(1 - \frac{1}{2k}
      \right)\frac{k}{T-1}\right) \log \frac{1 - \frac{k}{T-1}}{1 -
      \left(1 - \frac{1}{2k} \right)\frac{k}{T-1}} \\
}
    & = \left(1 - \frac{1}{2k} \right) \frac{k}{T-1} \log \frac{1}{1 - \frac{1}{2k}} + \left[ 1- \frac{k}{T-1} + \frac{1}{2k} \frac{k}{T-1} \right] \log \left[1 - \frac{\frac{1}{2k}\frac{k}{T-1} }{1 - \frac{k}{T-1} + \frac{1}{2k} \frac{k}{T-1}} \right] \\
    & \leq \left[ \frac{1}{2k} - \frac{1}{8k^2} \right] \frac{k}{T-1}
      + \left[ 1 - \frac{k}{T-1} + \frac{1}{2k} \frac{k}{T-1} \right]
      \frac{-\frac{1}{2k} \frac{k}{T-1}}{1 - \frac{k}{T-1} +
      \frac{1}{2k} \frac{k}{T-1}} \\
    & = -\frac{ \frac{1}{4k^2} \frac{k}{T-1} }{2} = -\frac{1}{8k (T-1)}.
\end{align*}
Using these inequalities, we can further bound the approximation error
in the regret bound by:
\begin{align*}
    \log\left[ \frac{1}{\sfq_{\sfw_2}[\bz \in \sC_T]} \right]
    & \leq \log\left[ \frac{1}{1 - 2e^{-\frac{1}{12k}}} \right]
    = - \log \left( 1 - 2e^{-\frac{1}{12k}} \right),
\end{align*}
which completes the proof.
\end{proof}

\newpage
\section{Time-independent approximation of competitor automata}
\label{app:timeindepapprox}
In the previous sections, we have introduced the technique of approximating
the automaton accepting competitor sequences of length $T$, $\sC_T$.
Intersecting $\sC$ with $\sS_T$ for different $T$ typically results in
different approximation automata.  Since each approximation requires
solving a convex optimization problem, this can become computationally
expensive.

\begin{figure}[t]
  \centering
  \includegraphics[scale=0.6]{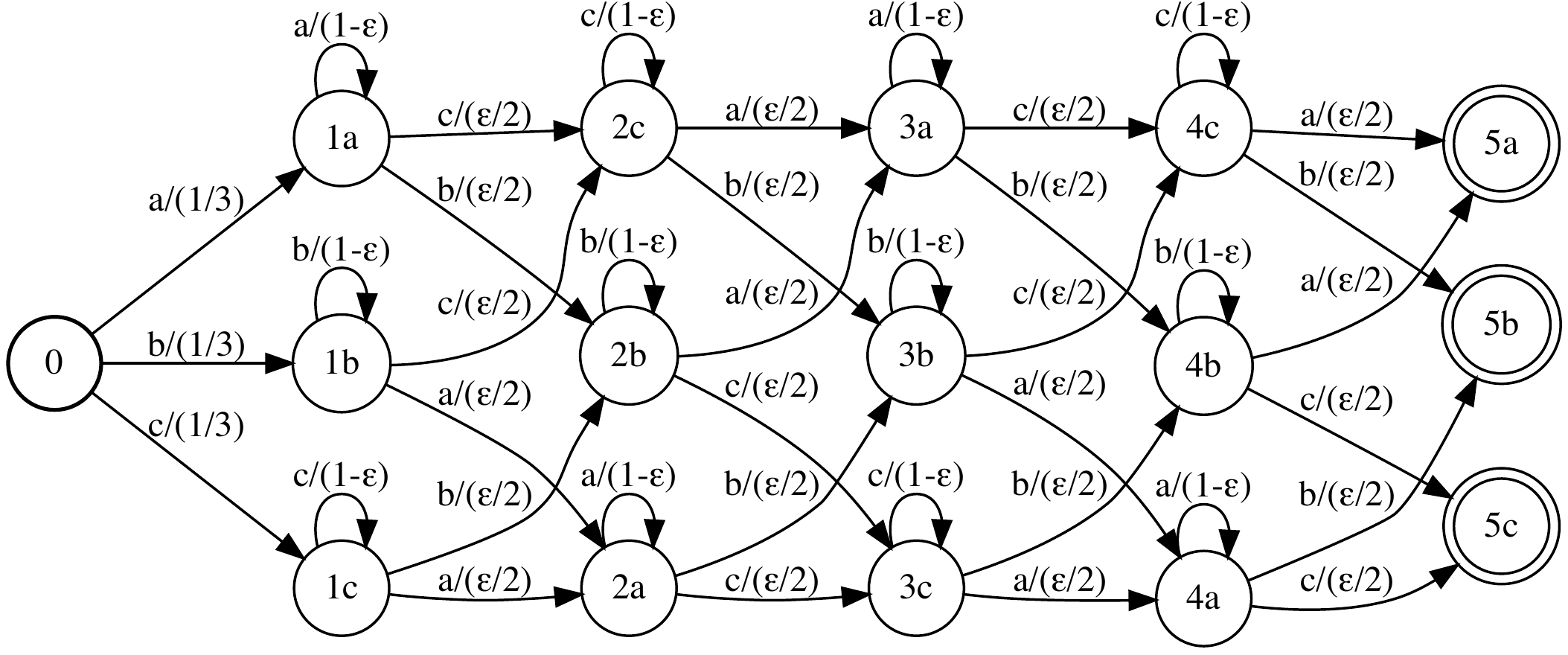} 
  \caption{Illustration of $\sC_{\text{$k$-shift},\e}$, where $k=4$ and $\Sigma=\{a,b,c\}$. 
  This automaton accepts an infinite number of paths, and the weights have been
  rescaled from the original $k$-shifting automaton so that the path weights 
  are now summable.}  \label{fig:kshifteps}
\end{figure}

In this section, we show how one can approximate the competitor set
for different $T$ using a single approximation. The key is to
approximate the original automaton $\sC$ directly. Specifically,
assume first that $\sC$ is a stochastic automaton (so that its
outgoing transition weights at each state sum to $1$), and let $\cC$
be a family of distributions over $\Sigma^*$ that we will use to
approximate $\sC$. Given, $\widehat{\sC} \in \cC$, define for every
$\bx \in \Sigma^*$
\[
  \tilde{\sfq}_{\widehat{\sC}}[\bx] = 
  \begin{cases}
    \widehat{\sC}[\bx] \frac{\sC[\sS_{|\bx|}]}{\widehat{\sC}[\sS_{|\bx|}]} & \text{ if } \widehat{\sC}[\sS_{|\bx|}] > 0 \\
    0 & \text{ otherwise }
  \end{cases} 
\]
Thus, $\tilde{\sfq}_{\widehat{\sC}}$ is a 
rescaling of $\widehat{\sC}$ based on the mass assigned by $\sC$ to sequences
of length equal to $|\bx|$.
Note that $\tilde{\sfq}_{\widehat{\sC}}$ may not necessarily be a 
distribution.  
Our algorithm consists of determining the 
best approximation to the competitor distribution  
$\sC$ within the family of rescaled distributions: 
\begin{align}
  \min_{\widehat{\sC}\in \cC} D_\infty(\sfq_\sC \| \tilde{\sfq}_{\widehat{\sC}}).
\end{align}
Note that this is an implicit extension of the definition of 
$\infty$-R\'{e}nyi divergence,
since $\tilde{\sfq}_{\widehat{\sfq}}$ may not be a distribution. 

The design of this optimization problem is motivated by the following result,
which guarantees that if $\tilde{\sfq}_{\widehat{\sC}}$ is a good approximation of $\sC$,
then $\sfq_{\widehat{\sC} \cap \sS_T}$ will be a good approximation
of $\sC \cap \sS_T$ for any $T$.

\begin{theorem}
  For any stochastic automata $\sC$ and $\widehat{\sC}$, and for any $T \geq 1$, 
  $$D_\infty(\sfq \| \sfq_{\widehat{\sC} \cap \sS_T}) \leq D_\infty(\sfq_\sC \| \tilde{\sfp}_{\widehat{\sC}}).$$ 
\end{theorem}
\begin{proof}
  Let $\bx \in \sC_T = \sC \cap \sS_T$ such that $(\sC \cap \sS_T)[\bx] > 0$.
  Since $(\sC \cap \sS_T)[\bx] = \frac{\sC[\bx]}{\sC[\sS_T]}$, this implies that 
  $\sC[\sS_T] \geq \sC[\bx] > 0$. Thus, if $\sfq_{\widehat{\sC}}[\sS_T] > 0$, then 
  \begin{align*}
    \log\left(\frac{(\sC \cap \sS_T)[\bx]}{ \sfq_{\widehat{\sC} \cap \sS_T}[\bx] }\right)
    & = \log\left(\frac{\sC[\bx] \widehat{\sC}[\sS_T]}{ \sC[\sS_T] \widehat{\sC}[\bx]} \right)
    =  \log\left(\frac{\sC[\bx] }{ \tilde{\sfq}_{\widehat{\sC}}[\bx]} \right).
  \end{align*}
  On the other hand, if $\widehat{\sC}(\sS_T) = 0$, then, by definition,
  $\tilde{\sfq}_{\widehat{\sC}}[\bx] = 0$, therefore the following inequality holds 
  \begin{align*}
    \log\left(\frac{(\sC \cap \sS_T)[\bx]}{ \sfq_{\widehat{\sC}\cap \sS_T}[\bx] }\right)
    &\leq \infty
    = \log\left(\frac{\sC[\bx] }{ \tilde{\sfq}_{\widehat{\sC}}[\bx]} \right).
  \end{align*}
  The result now follows by taking the maximum over $\bx \in \sS_T$ on the left-hand side
  and the maximum over $\bx \in \Sigma^*$ on the right-hand side.
\end{proof}

Note that, for $n$-gram approximations, $\sfq_{\sfw} \in \cW_n$, the condition
$\sfq_{\sfw}(\sS_T) = 1$ always holds. Thus, the approximation optimization problem
can be written as: 
\begin{align*}
  \min_{\sfw \in \cW_n} D_\infty(\sfq_\sC \| \tilde{\sfq}_{\sfw})
  & = \min_{\sfw \in \cW_n} \sup_{\bx \in \sC} \log\left( \frac{\sfq_\sC[\bx]}{\sfq_{\sfw}[\bx]\sfq_\sC[\sS_{|\bx|}]}\right). 
\end{align*}
As in Section~\ref{sec:min-Renyi}, this problem is the minimization of the 
supremum of a family of convex functions over the product of simplices. Thus, 
it is a 
convex optimization problem and can be solved using the \textsc{Prod-EG} algorithm. 

We have thus far assumed that $\sC$ is a stochastic automaton in this 
section. If 
the sum of the weights of all paths accepted by $\sC$ is finite, we can
apply weight-pushing to normalize the automaton to make it stochastic 
and then solve the approximation problem above. 

However, this property may not always hold. For example, the original $k$-shifting
automaton shown in Figure~\ref{fig:kshift} accepts an infinite number of paths (sequences of
arbitrary length with $k$ shifts). Since each transition has unit weight, 
each path also has unit weight, and the sum of the weight of all paths is infinite.

However, we can still apply the approximation method in this section to the $k$-shifting
automaton by rescaling the transitions weights of self-loops to be less than $1$.
Specifically, consider the automaton $\sC_{\text{$k$-shift}, \e}$
whose states and transitions are exactly the same as those of the original
automaton $\sC_{\text{$k$-shift}}$, except that transitions
from $ta_j$ to $(t+1)a_{k}$ for $a_j \neq a_{k}$ now have weight 
$\frac{\e}{N+1}$, and self-loops now have weight $1-\e$. To make the automaton
stochastic, we also assign weight $\frac{1}{N}$ to every initial state.
Then, the weight of a sequence of length $T$ accepted by $\sC_{\text{$k$-shift},\e}$
is $(1-\e)^{T-k-1} \left(\frac{\e}{N-1}\right)^k \frac{1}{N},$
and the weight of all sequences is finite.

By normalizing the weights of this automaton, we can convert it into a 
stochastic automaton, where 
$$\sfq_{\sC_{\text{$k$-shift},\e}}[\bx] \propto   (1-\e)^{|\bx|-k-1} \left(\frac{\e}{N-1}\right)^k \frac{1}{N}.$$
Figure~\ref{fig:kshifteps} shows the weighted automaton $\sC_{\text{$k$-shift},\e}$.

To compare with the results in Section~\ref{sec:approx}, we will now 
analyze the approximation error of a maximum-likelihood-based 
bigram approximation. 

\begin{theorem}[Bigram approximation of $\sC_{\text{$k$-shift},\e}$]
  \label{th:bigramkshifteps}
  The maximum-likelihood based bigram model for $\sC_{\text{$k$-shift},\e}$ is
  defined by 
  $$\sfq_{\sfw_2}[z_2 | z_1] = \frac{ \sum_{\tilde{T} \geq k +1 }  (1-\e)^{\tilde{T}-k-1}
  \left(1_{z_1 \neq z_2} \frac{k}{\tilde{T}-1} \frac{1}{N-1} + 1_{z_1 = z_2} \left( 1- \frac{k}{\tilde{T}-1}\right) \right)
  }{ \sum_{\tilde{T} \geq k + 1} (1-\e)^{\tilde{T}-k-1}}.$$
  Moreover, for every $T > k + 1$, there exists $\e \in (0,1)$ such that
  $$D_\infty(\sfq \| \sfq_{\sfw_2}) \leq - \log \left( 1 - 2e^{-\frac{1}{12k}} \right).$$
\end{theorem}
\begin{proof}
The maximum-likelihood $n$-gram automaton
is derived from the expected counts of the original automaton. Thus, 
for any $z_1, z_2 \in \Sigma$,
\begin{align*}
\sfq_{\sfw_2}[z_2 | z_1] 
  & = \frac{\sum_{\bx \in \sC_{\text{$k$-shift}}} (1-\e)^{|\bx|-k-1} \left(\frac{\e}{N-1}\right)^k 
  \frac{1}{N} \sum_{t=2}^{|\bx|} 1_{z_1^2 = x_{t-1}^t} }{\sum_{\bx \in \Sigma^*} 
  (1-\e)^{|\bx|-k-1} \left(\frac{\e}{N-1}\right)^k \frac{1}{N} \sum_{t=2}^{|\bx|} 
  1_{z_1 = x_{t-1}}} \\
  & = \frac{\sum_{\tilde{T} \geq k +1 } \sum_{\bx \in \sC_{\text{$k$-shift},\tilde{T}}} (1-\e)^{\tilde{T}-k-1} \sum_{t=2}^{\tilde{T}} 
  1_{z_1^2 = x_{t-1}^t} }{\sum_{\tilde{T} \geq k + 1}\sum_{\bx \in  \sC_{\text{$k$-shift},\tilde{T}}} (1-\e)^{\tilde{T}-k-1}  
  \sum_{t=2}^{\tilde{T}} 1_{z_1 = x_{t-1}}} 
\end{align*}
Now, notice that for any $\tilde{T}$, 
\begin{align*}
  &\sum_{\bx \in \sC_{\text{$k$-shift},\tilde{T}}}  
  \sum_{t=2}^{\tilde{T}} 1_{z_1^2 = x_{t-1}^t} \\
  &\quad = \sum_{\bx \in \sC_{\text{$k$-shift},\tilde{T}}} 
  \sum_{t=2}^{\tilde{T}}  1_{z_1 = x_{t-1}} 
  \left(1_{z_1 \neq z_2} \frac{k}{(\tilde{T}-1)(N-1)}  + 1_{z_1 = z_2} \left( 1- \frac{k}{\tilde{T}-1}\right) \right).
\end{align*}
This allows us to rewrite the probability above as:
\begin{align*}
  \sfq_{\sfw_2}[z_2 | z_1]
  & = \frac{ \sum_{\tilde{T} \geq k +1 }  (1-\e)^{\tilde{T}-k-1}
  \left(1_{z_1 \neq z_2} \frac{k}{(\tilde{T}-1)(N-1)}  + 1_{z_1 = z_2} \left( 1- \frac{k}{\tilde{T}-1}\right) \right)
  }{ \sum_{\tilde{T} \geq k + 1} (1-\e)^{\tilde{T}-k-1}}.
\end{align*}
Thus, $\sfp_{\sA_2}[z_2 | z_1]$ depends only on the condition $z_1 \neq z_2$.

  Now, fix $T > k + 1$. Since for every $\bx \in \sC_{\text{$k$-shift}} \cap \sS_T = \sC_{\text{$k$-shift},T}$, $x$ has $k$ shifts and length $T$,
$\sfq_{\sfw_2}$ is uniform over all sequences in $\sC_{\text{$k$-shift},T}$. This allows
us to bound the $\infty$-R\'enyi divergence between $\sfq = \sfq_{\sC_{\text{$k$-shift},T}}$ and
$\sfq_{\sfw_2}$ by:
\begin{align*}
  &\sup_{\bx \in \sC_{\text{$k$-shift},T}} \log\left(\frac{\sfq_{\sC_{\text{$k$-shift},T}}[\bx] }{ \sfq_{\sfw_2}[\bx] } \right) \\
  & = \sup_{\bx \in \sC_{\text{$k$-shift},T}} \log\left(\frac{\sfq_{\sC_{\text{$k$-shift},T}}[\bx] }{ \sfq_{\sfw_2}[\xi  = x|\xi \in \sC_T] \sfq_{\sfw_2}[\xi \in \sC_T]
  + \sfq_{\sfw_2}[\xi = x | \xi \notin \sC_T]\sfq_{\sfw_2}[\xi \notin \sC_T]} \right) \\
  & \leq \sup_{\bx \in \sC_{\text{$k$-shift},T}} \log\left(\frac{1 }{ \sfq_{\sfw_2}[\xi \in \sC_T]} \right).
\end{align*}
If we now let $(\xi_t)_{t=2}^T$ denote i.i.d. Bernoulli random variables with
mean 
  $$\bar{p}(\e) = \frac{\sum_{\tilde{T}\geq k+1} (1-\e)^{\tilde{T}-k-1} \frac{k}{\tilde{T}-1}}{\sum_{\tilde{T}\geq k+1} (1-\e)^{\tilde{T}-k-1}},$$ 
then
\begin{align*}
\sfq_{\sfw_2}[\xi \notin \sC_T] 
  &\leq \bbP\left[ \left| \sum_{t=2}^T \xi_t - \bar{p}(\e)(T-1) \right| \geq \frac{1}{2} \right] + \bbP\left[ \left|  \bar{p}(\e)(T-1) - k \right| \geq \frac{1}{2} \right]. 
\end{align*}
Thus, if $|\bar{p}(\e)(T-1) - k | < \frac{1}{2}$, then $\sfq_{\sfw_2}[\xi \notin \sC_T]$
can be bounded using the same concentration argument as in Theorem~\ref{th:bigramkshift}.

$\bar{p}(\e)$ can be interpreted as the weighted average of $\frac{k}{\tilde{T}-1}$
for $\tilde{T} \geq k +1$, where the weight of $\frac{k}{\tilde{T}-1}$ is 
$(1-\e)^{\tilde{T}-k-1}$. We want this average to be close to $\frac{k}{T-1}$
for the specific choice of $T > k +1$, which we obtain by appropriately
tuning $\e \in (0,1)$.

Since $\lim_{\e \to 0^+} \bar{p}(\e)= 1$,  $\lim_{\e \to 1^-} \bar{p}(\e) = 0$
and $\bar{p}(\e)$ is continuous in $\e$ on $(0,1)$, it follows by the intermediate
value theorem that for any $T > k+1$, there exists
an $\e^*$ such that $\bar{p}(\e^*) = \frac{k}{T-1}$. 
\end{proof}

Note that in the proof of the above theorem, $\bar{p}(\e)$ is monotonic in $\e$. 
Thus, one can find $\e'$ such that 
$\left|\bar{p}(\e') - \frac{k}{T-1}\right| \leq \frac{1}{2(T-1)}$ 
using binary search.

\newpage
\section{Extension to sleeping experts}
\label{app:sleep}

\begin{algorithm2e}[t]
  \TitleOfAlgo{\textsc{AwakeAWM}($\sC$, $\eta$)}
  $\sB \gets \sC \cap \sS_T$ \\
  $\sA \gets \textsc{Weight-Pushing}(\sB^\eta)$ \\
  $\bbeta \gets \textsc{BwdDist}(\sA)$ \\
  $\balpha \gets 0$; $\balpha[I_\sA] \gets 1$ \\ 
  \ForEach{$e \in E_{\sA}^{0 \to 1}$}{
      $\sfp_1[\lab[e]] \gets \weight[e]$.
    }
  \For{$t \gets 1$ \KwTo $T$}{
    $\textsc{Receive}(A_t)$ \\
    \ForEach{$a \in A_t$}{
      $\sfp_t^{A}[a] \gets \sfp_t[a] / \sfp_t(A_t)$ \\
    }
    $i_t \gets $\textsc{Sample}($\sfp_t^{A}$);  \textsc{Play}($i_t$); \textsc{Receive}($\bfl_t$)\\
    $Z \gets 0$; $\bw \gets 0$; $Z^A \gets 0$ \\ 
    \ForEach{$e \in E_{\sA}^{t \to t + 1}$}{
      \If{$\lab[e] \in A_t$}{
        $\weight[e] \gets \weight[e] \, e^{-\eta l_t[\lab[e]]}$ \\
      }
      $\bw[\lab[e]] \gets \bw[\lab[e]] + \balpha[\src[e]] \, \weight[e] \, \bbeta[\dest[e]]$\\
      $\balpha[\dest[e]] \gets \balpha[\dest[e]] + \balpha[\src[e]] \, \weight[e] $ \\
      \If{$\lab[e] \in A_t$}{
        $Z^{A} \gets Z^A + \bw[\lab[e]]$ \\
      }
    }
    $\sfp_{t + 1} \gets \bw \frac{\sfp_t(A_t)}{Z^A}$ \\ 
  }
\caption{\textsc{AwakeAutomataWeightedMajority}(AwakeAWM).} 
\label{alg:awakeawm}
\end{algorithm2e}

\begin{reptheorem}{th:awakeawm}[Regret Bound for \textsc{AwakeAWM}]
  Let $K$ denote the number of accepting paths of $\sC_T = \sC \cap \sS_T$,
  and for each $t \in [T]$, let $A_t\subseteq \Sigma$ denote the set of experts 
  that are awake at time $t$.
  Then for any distribution $\sfu\in \Delta_K$, \textsc{AwakeAWM} admits
  the following unweighted regret guarantee:
  \begin{align*}
    \sum_{t = 1}^T \sum_{\bx \in \sC_T \cap A_t} \sfu[\bx] \E_{a \sim \sfp_t^{A_t}} [l_t[a]] - 
    \sum_{t=1}^T \sum_{\bx \in \sC_T \cap A_t} \sfu[\bx] l_t[\bx[t]]
    &\leq \frac{\eta}{8} \sum_{t = 1}^T \sfu(A_t) + \frac{1}{\eta} \log(K). 
  \end{align*}
\end{reptheorem}
\begin{proof}
  As in the proof of Theorem~\ref{th:awm}, for every $t \in [T]$ and 
  $\bx \in \Sigma^T$, let $w_t[\bx]$ denote the sequence weight defining
  $\sfq_t$ via normalization, 
  $\sfq_t[\bx] = \frac{w_t[\bx]}{\sum_{\bx} w_t[\bx]}$. Moreover,
  let $\sfq_t^{A_t}$ be the distribution induced over sequences in
  with labels that awake at time $t$, so that for every sequence $\bx \in \sC_T$
  with $\bx[t] \in A_t$, 
  $\sfq_t^{A_t}[\bx] = \frac{\sfq_t[\bx]}{\sum_{\bx \in \sC_T \colon \bx[t] \in A_t}\sfq_t[\bx]}$,
  and for every sequence $\bx \in \sC_T$ with $\bx[t] \notin A_t$, $\sfq_t^{A_t}[\bx] = 0$.
  
  Notice that by design, if a sequence $\bx \in \sC_T$ has a label that isn't 
  awake at time $t$,  $x[t] \notin A_t$, then $\sfq_{t+1}[\bx] = \sfq_t[\bx]$,
  since we do not update that edge.

  Moreover, by the normalization scheme, 
  $\sum_{\bx \in \sC_T \colon \bx[t] \notin A_t} \sfq_{t+1}[\bx] = \sum_{\bx \in \sC_T \colon \bx[t] \notin A_t} \sfq_t[x]$.

  Now let $\sfu \in \Delta_K$. Then we can write
  \begin{align*}
    &D(\sfu \| \sfq_t) - D(\sfu \| \sfq_{t+1})\\ 
    &= \sum_{\bx \in \sC_T} \sfu[\bx] \log \frac{\sfq_{t+1}[\bx]}{\sfq_t[\bx]} \\
    &= \sum_{x \in \sC_T \colon \bx[t] \in A_t} \sfu[\bx] \log\frac{\sfq_{t+1}[\bx]}{\sfq_t[\bx]} \\
    &= \sum_{x \in \sC_T \cap A_t} \sfu[\bx] \log\frac{\sfq_{t+1}^{A_t}[\bx]}{\sfq_t^{A_t}[\bx]} \\
    &= \sum_{x \in \sC_T \cap A_t} \sfu[\bx] \log\frac{\sfq_{t}^{A_t}[\bx]e^{-\eta l_t[\bx[t]]}}{\sfq_t^{A_t}[\bx] \sum_{\by \in \sC_T \colon \by[t] \in A_t} \sfq_{t}^{A_t}[\by]e^{-\eta l_t[\by[t]]}} \\
    &= \sum_{\bx \in \sC_T \cap A_t} \sfu[\bx] (-\eta l_t[\bx[t]]) - \sum_{\bx \in \sC_T \colon \bx[t] \in A_t} \sfu[\bx] \log\left( \sum_{\by \in \sC_T \colon \by[t] \in A_t} \sfq_{t}^{A_t}[\by]e^{-\eta l_t[\by[t]]}\right) \\
    &\leq -\eta \sum_{\bx \in \sC_T \cap A_t} \sfu[\bx] l_t[\bx[t]] - \sum_{\bx \in \sC_T \colon \bx[t] \in A_t} \sfu[\bx] \left( \E_{\by \sim \sfq_t^{A_t}} [-\eta l_t[\by[t]]] + \frac{\eta^2}{8} \right) \\
    &= -\eta \sum_{\bx \in \sC_T \cap A_t} \sfu[\bx] l_t[\bx[t]]  + \eta \sum_{\bx \in \sC_T \colon \bx[t] \in A_t} \sfu[\bx] \E_{a \sim \sfp_t^{A_t}} [ l_t[a]]  -  \sfu(A_t) \frac{\eta^2}{8}.
  \end{align*}
  Thus, by rearranging terms and summing over $t$, it follows that
  \begin{align*}
    \sum_{t=1}^T \sum_{\bx \in \sC_T \cap A_t} \sfu[\bx] \E_{a \in \sfp_t^{A_t}} [\eta l_t[a]] 
    - \sum_{t=1}^T \sum_{\bx \in \sC_T \cap A_t} \sfu[\bx] l_t[\bx[t]]   \leq 
    \sum_{t=1}^T \sfu(A_t) \frac{\eta}{8} + D(\sfu \| \sfq_1), 
  \end{align*}
  and since for the unweighted regret, $\sfq_1 = \frac{1}{K}$, $D(\sfu \| \sfq_1) \leq \log(K)$, which completes the proof.
\end{proof}

\newpage
\section{Extension to online convex optimization}
\label{app:oco}

We now show how the framework described in this paper can be extended
to the general online convex optimization (OCO) setting.  Online
convex optimization is a sequential prediction game over a compact
convex action space $\cK$. At each round $t$, the learner plays an
action $x_t \in \cK$ and receives a convex loss function $f_t$. The
goal of the learner is to minimize the regret against the best static
loss:
\begin{equation*}
\sum_{t = 1}^T f_t(x_t) - \min_{z \in \cK} \sum_{t = 1}^T f_t(z).
\end{equation*}
As in the framework we introduced, we can generalize
this notion of regret to one against a family of sequences. 
Specifically, let $\cC_T \subseteq \cK^T$ be a closed subset, let 
$\sfq_{\cC_T}$ be a distribution over $\cC_T$, and 
and let $\sfu_{\cC_T}$ be the uniform distribution over $\cC_T$. The uniform distribution
is well-defined, since $\cK$ is a compact set implies and thus $\cK^T$
as well. 
Then we would like to compete against the following regret against $\sfq_{\cC_T}$:
\begin{align}
  \label{eq:ocoregret}
  \Reg_T(\cA, \sC_T) = \max_{z_1^T \in \cC_T} \sum_{t = 1}^T f_t(x_t) - f_t(z_t) + \log\left[\frac{\sfq_{\cC_T}(z_1^T)}{\sfu_{\cC_T}(z_1^T)}\right].
\end{align}
When $\sfq_{\cC_T}$ is uniform, the last term vanishes.  When
$\cC_T = \cK^T$ is the family of all sequences of length $T$ and
$\sfq_{\cC_T}$ is the uniform distribution, this problem has been
studied in \citep{HallWillett2013, GyorgySzepesvari2016}.  In both
works, the authors introduce a variant of mirror descent that applies
a mapping after the standard mirror descent update and which is called
\textsc{DynamicMirrorDescent} in the first paper.

Specifically, if $g_t \in \partial f_t(x_t)$ is an element of the subgradient,
and $D_\psi$ is the Bregman divergence induced by a mirror map $\psi$,
then \textsc{DynamicMirrorDescent} consists of the following update rule:
\begin{align*}
  &\tilde{x}_{t+1} \gets \argmin_{x \in \cK} \langle g_t, x\rangle + D_\psi(x, x_t) \\
  &x_{t+1} \gets \Phi_t(\tilde{x}_{t+1}).
\end{align*}
In this algorithm, $\Phi_t$ is an arbitrary mapping that is specified
by the learner at time $t$.  Under certain assumptions on the loss
functions $\Psi$ and $\Phi_t$, we can show that \textsc{DynamicMirrorDescent} achieves
the following regret guarantee against the competitor distribution
$\sfq$:

\begin{theorem}[\textsc{DynamicMirrorDescent} regret against $\cC_T$]
  \label{th:ocobound}
  Suppose that the $\Phi_t$s chosen in \textsc{DynamicMirrorDescent} are non-expansive 
  under the Bregman divergence $D_\psi$:
$$D_\psi(\Phi_t(x), \Phi_t(y)) \leq D_\psi(x,y), \quad \forall x,y\in \cK.$$
Furthermore, assume that $f_t$ is uniformly $L$-Lipschitz in the norm $\|\cdot\|$ and that $\Psi$ is $1$-strongly convex
  in the same norm. Let $x_1 \in \cK$ be given and define 
  $D_{\max} = \sup_{z\in \cK} D_\psi(z, x_1)$. 
  Then, \textsc{DynamicMirrorDescent} 
  achieves the following regret guarantee:
\begin{align*}
  \Reg_T(\cA, \sC_T)
  &\leq\frac{D_{\max}}{\eta} + \frac{\eta}{2} \sum_{t = 1}^T \|g_t\|^2 +  \max_{z_1^T \in \cC_T} \Big\{\log\left[\frac{\sfq_{\cC_T}(z_1^T)}{\sfu_{\cC_T}}\right]\\
  &\quad + \frac{2}{\eta} \sum_{t = 1}^T \psi(z_{t+1}) - \psi(\Phi_t(z_t)) - \langle \nabla \psi(x_{t+1}), z_{t+1} - \Phi_t(z_t) \rangle \Big\}.
\end{align*}
\end{theorem}
This result can be proven using similar ideas as in \citep{HallWillett2013}. 
The main difference is that
\citet{HallWillett2013} assume $\Psi$ to be Lipschitz. This allows
them to derive a slightly weaker but more interpretable
bound. However, it is also an assumption that we specifically choose
to avoid, since mirror descent algorithms including the
\textsc{Exponentiated Gradient} use mirror maps that are not
Lipschitz. \citet{HallWillett2013} also derive a bound for standard
regret as opposed to regret against a distribution of sequences.

The first two terms in the regret bound are standard in online convex 
optimization, and the last term
is the price of competing against arbitrary sequences.
Note that \citet{GyorgySzepesvari2016} present the same algorithm
but with a different analysis and upper bound.

\begin{proof}
  By standard properties of the Bregman divergence and convexity, we can compute
  \begin{align*}
    \sum_{t = 1}^T f_t(x_t) - f_t(z_t)
    & = \sum_{t = 1}^T f_t(x_t) - f_t(z_t) + f_t(\tilde{x}_{t+1})  - f_t(\tilde{x}_{t+1}) \\
    &\leq \frac{1}{\eta} \langle \nabla \psi(x_t) - \nabla \psi(\tilde{x}_{t+1}), \tilde{x}_{t+1} - z_t \rangle + f_t(x_t) - f_t(\tilde{x}_{t+1}) \\
    & = \frac{1}{\eta} \left[ D_\psi(z_t, x_t) - D_\psi(z_t, \tilde{x}_{t+1}) - D_\psi(\tilde{x}_{t+1}, x_t) \right] + f_t(x_t) - f_t(\tilde{x}_{t+1}) \\
    & = \frac{1}{\eta} \Big[ D_\psi(z_t, x_t) - D_\psi(z_{t+1}, x_{t+1}) + D_\psi(z_{t+1}, x_{t+1} - D_\psi(\Phi_t(z_t), x_{t+1}) \\
    &\quad - D_\psi(z_t, \tilde{x}_{t+1}) + D_\psi(\Phi_t(z_t), x_{t+1}) - D_\psi(\tilde{x}_{t+1}, x_t)\Big] \\
    &\quad + f_t(x_t) - f_t(\tilde{x}_{t+1}). 
  \end{align*}
  Since $\Phi_t$ is assumed to be non-expansive and $x_{t+1} = \Phi_t(x_{t+1})$, it follows that $ - D_\psi(z_t, \tilde{x}_{t+1}) + D_\psi(\Phi_t(z_t), x_{t+1} \leq 0$.
 
  Since $\Psi$ is $1$-strongly convex with respect to $\|\cdot\|$, it follows that
  $D_\psi(\tilde{x}_{t+1}, x_t) \geq \frac{1}{2} \|\tilde{x}_{t+1} - x_t\|^2$. Thus, we can compute 
  \begin{align*}
    &- \frac{1}{\eta} D_\psi(\tilde{x}_{t+1}, x_t) + f_t(x_t) - f_t(\tilde{x}_{t+1}) \\
    &\leq -\frac{1}{2\eta} \|\tilde{x}_{t+1} - x_t\|^2 + f_t(x_t) - f_t(\tilde{x}_{t+1}) \\
    &\leq -\frac{1}{2\eta} \|\tilde{x}_{t+1} - x_t\|^2 + \|g_t\|_* \|x_t - \tilde{x}_{t+1}\| \\
    &\leq -\frac{1}{2\eta} \|\tilde{x}_{t+1} - x_t\|^2 +\frac{\eta}{2} \|g_t\|_*^2 + \frac{1}{2\eta} \|x_t - \tilde{x}_{t+1}\|^2 \\
    & = \frac{\eta}{2} \|g_t\|_*^2.
  \end{align*}
  Moreover, we can also write
  \begin{align*}
    &D_\psi(z_{t+1}, x_{t+1} - D_\psi(\Phi_t(z_t), x_{t+1}) \\
    & = \psi(z_{t+1}) - \psi(x_{t+1}) - \langle \nabla \psi(x_{t+1}), z_{t+1} - x_{t+1} \rangle \\
    &\quad \left[ \psi(\Phi_t(z_t)) - \psi(x_{t+1}) - \langle \nabla \psi(x_{t+1}), \Phi_t(z_t) - x_{t+1} \rangle \right]\\
    & = \psi(z_{t+1}) - \psi(\Phi_t(z_t)) - \langle \psi(x_{t+1}), z_{t+1} - \Phi_t(z_t) \rangle.
  \end{align*}
  Combining this inequality with the inequality above yields
  \begin{align*}
    &\sum_{t = 1}^T f_t(x_t) - f_t(z_t) \\
    &\leq \frac{1}{\eta} \sum_{t = 1}^T  D_\psi(z_t, x_t) - D_\psi(z_{t+1}, x_{t+1}) \\
    &\quad + \sum_{t = 1}^T \psi(z_{t+1}) - \psi(\Phi_t(z_t)) - \langle \psi(x_{t+1}), z_{t+1} - \Phi_t(z_t) \rangle + \frac{1}{\eta} \sum_{t = 1}^T \frac{\eta}{2} \|g_t\|_*^2 \\
    &\leq \frac{1}{\eta} D_\psi(z_1,x_1)  + \sum_{t = 1}^T \psi(z_{t+1}) - \psi(\Phi_t(z_t)) - \langle \psi(x_{t+1}), z_{t+1} - \Phi_t(z_t) \rangle \\
    &\quad + \frac{1}{\eta} \sum_{t = 1}^T \frac{\eta}{2} \|g_t\|_*^2. 
  \end{align*}
  Adding in $\log\left(\frac{\sfq_{\cC_T}(z_1^T)}{\sfu_{\cC_T}(z_1^T)}\right)$ to both sides and 
  taking the max over $z_1^T \in \cC_T$ completes the proof.
\end{proof}

By restricting our competitor set to $\sC_T$ and adding the penalization term, it follows that
\textsc{DynamicMirrorDescent} achieves the following guarantee:
\begin{align*}
  &\max_{z_1^T \in \cC_T} \sum_{t = 1}^T f_t(x_t) - f_t(z_t) + \log\left[\frac{\sfq_{\cC_T}(z_1^T)}{\sfu_{\cC_T}(z_1^T)}\right]  \\
  &\quad \leq \frac{D_{\max}}{\eta} + \frac{\eta}{2} \sum_{t = 1}^T \|g_t\|^2 + \max_{z_1^T \in \cC_T} \Bigg\{ \log\left[\frac{\sfq_{\cC_T}(z_1^T)}{\sfu_{\cC_T}(z_1^T)}\right] \\
  &\quad \quad + \frac{2}{\eta} \sum_{t = 1}^T \psi(z_{t+1}) - \psi(\Phi_t(z_t)) - \langle \nabla \psi(x_{t+1}), z_{t+1} - \Phi_t(z_t) \rangle \Bigg\}.
\end{align*}
This bound suggests that if we could find a sequence
$(\Phi_t)_{t = 1}^T$ that minimizes the last quantity, then we could
tightly bound our regret. Now, let $\cF$ be a family of dynamic maps
$\Phi$ that are non-expansive with respect to $D_\psi$.  Then we want
to solve the following optimization problem:
\begin{align}
\label{eq:doco}
  & \min_{\Phi_1^T \in \cF^T} \max_{z_1^T \in \cC_T} \Bigg\{ \log\left(\frac{\sfq_{\cC_T}(z_1^T)}{\sfu_{\cC_T}(z_1^T)}\right) \nonumber \\ 
  & \quad + \frac{2}{\eta} \sum_{t = 1}^T \psi(z_{t+1}) - \psi(\Phi_t(z_t)) - \langle \nabla \psi(x_{t+1}), z_{t+1} - \Phi_t(z_t) \rangle \Bigg\}.
\end{align}
We can view this as the online convex optimization analogue of the
automata approximation problem in Section~\ref{sec:approx}, and we
can use it in the same way to derive concrete online convex
optimization algorithms that achieve good regret against more complex
families of sequences.

As an illustrative example, we apply this to the $k$-shifting experts
setting and show how a candidate solution to this problem recovers the
\textsc{Fixed-Share} algorithm.

{\bf OCO derivation of \textsc{Fixed-Share}}.  Suppose that we are
again in the prediction with expert advice setting so that
$\cK = \Delta_N$ and $f_t(x) = \langle l_t, x \rangle$. Assume that
$\sC_T$ is the set of $k$-shifting experts and that $\sfq$ is the
uniform distribution on $\sC_T$. As for the weighted majority
algorithm, let $\Psi = \sum_{i=1}^N x_i \log(x_i)$ be the negative
entropy so that
$D_\psi(x,y) = \sum_{i=1}^N x_i \log\left(\frac{x_i}{y_i}\right)$ is
the relative entropy. One way of ensuring that $\Phi_t$ is
non-expansive is to define it to be a mixture with a fixed vector:
$\Phi_t(x) = (1-\alpha_t) x + \alpha_t w_t$ for some
$w_t \in \Delta_N$ and $\alpha_t \in [0,1]$.  By convexity of the
relative entropy, it follows that for any $x,y \in \Delta_N$,
$D_\psi(\Phi_t(x), \Phi_t(y)) \leq D_\psi(x,y)$.

For simplicity, we can assume that $\Phi_t = \Phi$. Then Problem~\ref{eq:doco} 
can be written as:
\begin{align*}
  \label{eq:docokshift}
  \min_{\sfw \in \Delta_N, \alpha \in [0,1] } \max_{z_1^T \in \sC_T} 
  &\Bigg\{ \frac{2}{\eta} \sum_{t = 1}^T \psi(z_{t+1}) - \psi((1-\alpha) z_t + \alpha w) \\
  &\quad - \langle \nabla \psi((1-\alpha) \tilde{x}_{t+1} + \alpha w), z_{t+1} - (1-\alpha) z_t - \alpha w \rangle \Bigg\}.
\end{align*}
Since $\sC_T$ is symmetric across coordinates and we do not have a priori knowledge of 
of $\tilde{x}_{t+1}$, a reasonable choice of $w$ is the uniform distribution
$w_i = \frac{1}{N}$. 
We can also use the fact that the entropy function is convex to obtain the 
upper bound:
$-\psi((1-\alpha)z_t + \alpha w) \leq -(1-\alpha) \psi(z) - \alpha \psi(w)$. 
Moreover, since $z_{t}$ is always only supported on a single coordinate, $\psi(z_t) = 0$ for every $t$.

This reduces to the following optimization problem: 
\begin{align*}
  \min_{\alpha \in [0,1] } \max_{z_1^T \in \sC_T} 
  &\Bigg\{ \frac{2}{\eta} \sum_{t = 1}^T \alpha \log(N)\\ 
  &\quad - \sum_{i=1}^N \log\left((1-\alpha) \tilde{x}_{t+1, i} + \alpha \frac{1}{N}\right) \left[ z_{t+1,i} - (1-\alpha) z_{t,i} - \alpha \frac{1}{N} \right] \Bigg\}.
\end{align*}
 We can break the objective into three separate terms:
 \begin{align*}
   &A_1: \frac{2}{\eta} \sum_{t = 1}^T \alpha \log(N) \\
   &A_2:  -\sum_{t = 1}^T \sum_{i=1}^N \log\left((1-\alpha) \tilde{x}_{t+1, i} + \alpha \frac{1}{N}\right) \left[ z_{t+1,i} - z_{t,i} \right] \\
   &A_3: - \sum_{t = 1}^T \sum_{i=1}^N \log\left((1-\alpha) \tilde{x}_{t+1, i} + \alpha \frac{1}{N}\right) \alpha \left[ z_{t,i} - \frac{1}{N} \right] 
 \end{align*}

It is straightforward to see that $A_1 = \frac{2}{\eta}T \alpha \log(N)$.
To bound $A_2$, let $i_t \in [N]$ be the index such that $z_{t,i_t} = 1$ and $z_{t,i} = 0$ for all $i \neq i_t$.
Then,
\begin{align*}
  &-\sum_{t = 1}^T \sum_{i=1}^N \log\left((1-\alpha) \tilde{x}_{t+1, i} + \alpha \frac{1}{N}\right) \left[ z_{t+1,i} - z_{t,i} \right] \\
  & = - \sum_{t: i_{t+1} \neq i_t} \sum_{i=1}^N \log\left((1-\alpha) \tilde{x}_{t+1, i} + \alpha \frac{1}{N}\right) \left[ z_{t+1,i} - z_{t,i} \right] \\
  &\leq - \sum_{t: i_{t+1} \neq i_t} \sum_{i=1}^N \log\left( \alpha \frac{1}{N}\right)  z_{t+1,i} \\
  &\leq  - k \log\left(\frac{\alpha}{N}\right).
\end{align*}
To bound $A_3$, we can write
\begin{align*}
  &-\alpha \sum_{t = 1}^T \sum_{i=1}^N \log\left((1-\alpha) \tilde{x}_{t+1, i} + \alpha \frac{1}{N}\right) \left[ z_{t,i} - \frac{1}{N} \right] \\
  & =  -\alpha \sum_{t = 1}^T \log\left((1-\alpha) \tilde{x}_{t+1, i_t} + \alpha \frac{1}{N}\right) \left[ 1  - \frac{1}{N} \right]\\
  &\leq  -\alpha T \log\left(\alpha \frac{1}{N}\right). 
\end{align*}
Putting the pieces together, the objective is bounded by
$$\frac{2}{\eta} \left( T \alpha \log(N) - k \log\left(\frac{\alpha}{N}\right) - \alpha T \log \left(\frac{\alpha}{N}\right) \right),$$
leading to the new optimization problem:
\begin{align*}
  \min_{\alpha \in [0,1] } \frac{2}{\eta} \left( T \alpha \log(N) - k \log\left(\frac{\alpha}{N}\right) - \alpha T \log \left(\frac{\alpha}{N}\right) \right).
\end{align*}
Notice that $\alpha \propto \frac{k}{T}$ is a reasonable solution, as it 
bounds the regret by
$\cO\left( k \log\left(\frac{NT}{k}\right)\right)$.

Moreover, this choice of $\alpha$ approximately corresponds to
\textsc{Fixed-Share}.  Thus, we have again derived the
\textsc{Fixed-Share} algorithm from first principles in consideration
of only the $k$-shifting expert sequences.  This is in contrast with
previous work for \textsc{DynamicMirrorDescent}
(e.g. \citep{GyorgySzepesvari2016}) which only showed that one
could define $\Phi_t$ in a way that mimics the \textsc{Fixed-Share}
algorithm.

\end{document}